\documentclass{article}
\usepackage[utf8]{inputenc}

\usepackage[T1]{fontenc}
\usepackage{fullpage}
\usepackage{amsmath}
\usepackage{mathtools}
\usepackage[dvipsnames]{xcolor}
\usepackage{subcaption,tikz}
\usetikzlibrary{matrix,arrows,decorations.pathmorphing}
\usepackage{amssymb,mathrsfs}
\usepackage{bbm}
\usepackage{xr}
\usepackage{comment}
\usepackage{ulem}
\usepackage{amsthm}
\usepackage[all,cmtip]{xy}
\usepackage[toc,page]{appendix}
\usepackage{tabularx}
\newcolumntype{Y}{>{\centering\arraybackslash}X}
\usepackage{url}
\usepackage[british]{babel}
\usepackage{courier}
\usepackage{enumerate}
\usepackage[T1]{fontenc}
\usepackage{listings}
\usepackage{algorithm}
\usepackage[noend]{algpseudocode}
\usepackage{titling}
\usepackage{authblk}
\usepackage{cite}
\usepackage{siunitx}
\usepackage{ulem}
\newcommand{\vfx}[1]{\mathbf{X}_{#1}}
\newcommand{\vfy}[1]{\mathbf{Y}_{#1}}
\newcommand{\vfxx}{\mathbf{X}}
\newcommand{\vfyy}{\mathbf{Y}}
\newcommand{\vfww}{\mathbf{W}}
\theoremstyle{plain}
\newtheorem{theorem}{Theorem}[section]

\theoremstyle{definition}

\makeatletter
\newcommand{\keywords}[1]{%
  \par\addvspace{0.5em}
  \noindent
  \begin{list}{}%
    {\setlength{\leftmargin}{\leftmargini}
     \setlength{\rightmargin}{\leftmargini}
     \setlength{\listparindent}{0pt}
     \setlength{\itemindent}{0pt}}%
  \item[]\small\textbf{Keywords: }#1
  \end{list}
}
\makeatother
\title{Pattern recognition in complex systems via vector-field representations of spatio-temporal data}
\date{}

\author[1,2*]{Ingrid~Amaranta~Membrillo~Solis}
\author[3]{Maria~van~Rossem}
\author[2]{Tristan~Madeleine}
\author[3]{Nina~Podoliak}
\author[3]{Tetiana~Orlova}
\stepcounter{footnote}
\author[2]{Giampaolo~D'Alessandro\footnote{These authors jointly supervised this work.}}
\author[2]{Jacek~Brodzki\protect\footnotemark[2]}
\author[3]{Malgosia~Kaczmarek\protect\footnotemark[2]}
\affil[1]{School of Mathematical Sciences, Queen Mary University of London,  327 Mile End Road, E1 4NS, UK}
\affil[2]{Mathematical Sciences, University of Southampton,  Southampton SO17~1BJ, UK}
\affil[3]{Physics and Astronomy, University of Southampton,  Southampton SO17~1BJ, UK}

\affil[*]{Corresponding author, i.a.membrillosolis@qmul.ac.uk}

\begin{document}

\maketitle

\begin{abstract}

A complex system comprises multiple interacting entities whose interdependencies form a unified whole, exhibiting emergent behaviours not present in individual components. Examples include the human brain, living cells, soft matter, Earth’s climate, ecosystems, and the economy. These systems exhibit high-dimensional, non-linear dynamics, making their modelling, classification, and prediction particularly challenging. Advances in information technology have enabled data-driven approaches to studying such systems. However, the sheer volume and complexity of spatio-temporal data often hinder traditional methods like dimensionality reduction, phase-space reconstruction, and attractor characterisation. This paper introduces a geometric framework for analysing spatio-temporal data from complex systems, grounded in the theory of vector fields over discrete measure spaces. We propose a two-parameter family of metrics suitable for data analysis and machine learning applications. The framework supports time-dependent images, image gradients, and real- or vector-valued functions defined on graphs and simplicial complexes. We validate our approach using data from numerical simulations of biological and physical systems on flat and curved domains. Our results show that the proposed metrics, combined with multidimensional scaling, effectively address key analytical challenges. They enable dimensionality reduction, mode decomposition, phase-space reconstruction, and attractor characterisation. Our findings offer a robust pathway for understanding complex dynamical systems, especially in contexts where traditional modelling is impractical but abundant experimental data are available.

\end{abstract}


\section{Introduction}
\label{sec:Intro}

Many systems across all scientific disciplines exhibit complexity. A complex system consists of entities showing collective behaviour distinct from that of its constituents, and its dynamics is typically characterised by feedback loops, self-organisation, chaotic attractors or emergence. Examples of complex systems include the human brain \cite{bassett2011}, living cells~\cite{Maayan2017}, soft matter materials \cite{zhou2019}, the Earth's  climate \cite{rind1999}, organisms \cite{camazine2020} and the economy \cite{beinhocker2006}.  Because of their ubiquity and rich behaviour, complex systems are important and challenging to study.  They share some common features. First, they are difficult to model and analyse because many constituent entities and interactions result in intricate, high-dimensional, non-linear dynamics. From a mathematical point of view, they may be represented by sets of differential equations, possibly stochastic, but also by discrete models, e.g. self-driven many-particle systems. The same complex system may be amenable to both descriptions, e.g. crowd dynamics~\cite{Helbing2000, Hughes2003} and active matter in general~\cite{Marchetti2013}.  From a data science point of view, their state can be represented by a range of formats of spatio-temporal data: movies (i.e. two-dimensional representations) and volume data sampled at discrete times, as well as more complex data structures. In biofilms~\cite{Penesyan2021}, for example, relevant data may be the positions of the biofilm components, also their gene, activity and stage in the reproductive cycle. In all these cases, the spatio-temporal data associated with these systems are ultimately $n$-tuples of numbers defined over discrete spatial structures, where each point in the spatial structure has associated a real or a vector value. From a mathematical point of view, this type of data can be described as \textit{vector fields over discrete measure spaces}. We will precisely define these mathematical objects in the next section.  A straightforward example of such data structure is a monochromatic image: such an image consists of a rectangular lattice in which every point in the lattice has been assigned a number, a one dimensional vector, that takes a value between 0 and 255.  In the soft-matter systems that motivated this study, skyrmion dynamics in liquid crystals~\cite{Sohn2019Schools} and cellular networks formation in liquid crystal doped with gold nanoparticles~\cite{Milette2012Reversible}, the data sets consist of movies.  We will report our analysis of these systems elsewhere.  Here, we aim to present a geometric-based method to analyse complex spatio-temporal data with the structure of vector fields over discrete metric measure spaces and apply it to analyse artificial data modelling complex system dynamics obtained by integrating numerically exemplar partial differential equations.

Data analysis from complex systems presents two challenges, and we discuss these in turn. The first challenge is that to achieve an accurate and computationally tractable prediction of dynamical systems, it is necessary to develop low-dimensional models that are easy to handle but provide a good approximation of the underlying dynamics \cite{ghadami2022}. A significant hurdle in studying complex system dynamics through data-driven approaches is the large amounts of complex spatio-temporal data generated, such as images, video recordings or networks,  which increases the difficulty of analysing and interpreting it. Information technology and artificial intelligence developments have brought new tools and approaches to detecting, classifying, understanding, predicting and controlling dynamical systems. Standard data-driven tools used in the analysis of dynamical systems include artificial neural networks \cite{champion2019, wan2018} and classical machine learning methods, such as clustering \cite{fernex2021} and classical principal component analysis (PCA) \cite{maisuradze2009}.

Classical PCA and proper orthogonal decomposition (POD) are some of the most frequently used methods for dimensionality reduction. \cite{Rega2005Dimension}. Classical PCA uses Euclidean metrics to quantify the similarity between spatio-temporal data points, which are seen as points in a high-dimensional vector space. However, classical PCA faces computational challenges when dealing with high-dimensional data sets \cite{fan2014}. Similarly, POD has been applied to detect the lowest dimensional reduced order models using linear coordinate transformation \cite{kerschen2002}. Nonlinear methods based on invariant manifolds have also been proposed \cite{touze2021}. Still, they are extremely challenging in general and even more so when the model is unknown and only data are available. Deep learning methods have recently been developed to produce reduced-order models \cite{lee2020, Li2021Hierarchical}. For an efficient application of deep learning methods in complex system dynamics, large amounts of data are required for generalised solutions, which may hinder the identification and prediction of many real-world dynamical systems. 
The second challenge is the difficulty of detecting the structure of the underlying attractor.  Complex systems in equilibrium may display chaotic dynamics.  Chaos theory has been a very active study area for more than 50 years~\cite{Olsen1985}. It has a wide range of applications, for example, in biology: examples of systems which are conjectured to exhibit chaos include heart beats, neural systems, population dynamics, and evolution~\cite{Toker2020}.  Its key features are the long-term unpredictability of orbits and the evolution of the equilibrium orbit over a "strange" attractor, a globally attractive sets of locally unstable orbits.  They are very different "standard" dynamical system attractors, like fixed points and periodic orbits, which are locally stable and lead to long-term predictability.  Distinguishing the type of attractor presented in a given complex dynamical system can be challenging. Various methods have been introduced to distinguish a system as chaotic or non-chaotic \cite{gottwald2004,djurovic2008,Toker2020}. Some apply only to low-dimensional attractors \cite{kantz2004}. Most of the others rely on binary detection tests that do not further characterise the system's attractor.   In \cite{bhattacharya2021} the authors use standard hidden Markov modeling, a supervised learning method, to detect and classify the type of attractor in a dynamical system. However, supervised methods will fail in cases where little or no knowledge is available about the nature of the complex dynamical system. 
Moreover, it is possible that a complex system never settles on an equilibrium attractor, or that the time series is not long enough for the transient to be over.  In this case, it may not be possible to classify the dynamics in terms of attractors, but it may still be beneficial to characterise it in terms of its transient states \cite{Vanhatalo2017structure}.

In this paper, we present a geometric approach to the analysis of spatio-temporal data generated in the study of complex dynamical systems, which can be efficiently used for dimensionality reduction, phase space reconstruction, mode decomposition and global attractor classification without the need for any previous knowledge on the nature of the dynamics of the system. Our framework is based on the observation that various spatio-temporal data can be regarded as vector fields over discrete measure spaces. We introduce a 2-parameter family of metrics on the set of vector fields over a discrete measure space, which permits distinguishing the physical states of a complex system at different levels of resolution.
Examples of data for which the theory of vector fields over discrete measure spaces can be applied are monochromatic, RGB or any other channel-type images, gradients of images, discrete vector fields over meshes, graphs and simplicial complexes, and their correspondent discrete gradients. We demonstrate that the metric spaces constructed from the data generated by the dynamical systems can be later used as the input data in the multidimensional scaling (MDS) algorithm for dimensionality reduction, dynamic mode decomposition and phase space reconstruction. Our methodology extracts dynamical information from extensive unstructured data such as high-resolution images at low computational cost. It is flexible in terms of the type of data that can be analysed through it (images, gradients, simplicial complexes, etc.). Finally, we encode it in an algorithm that permits us to approximate discretised vector-valued functions over curved (non-flat) domains using the first principal coordinates of the system. As a proof of concept, we apply our geometric framework to the data analysis of numerical solutions of the Ginzburg-Landau equation on flat domains, their gradient, and numerical solutions of reaction-diffusion systems on spherical domains. These are dynamical systems which can exhibit chaotic behaviour. In particular, one of the systems  considered in this work is an example of spatio-temporal chaotic dynamics in the context of morphogenesis. Our results demonstrate that the proposed analytic framework outlined in this paper allows one to reconstruct the low-dimensional phase spaces of the systems and to distinguish between different types of dynamical behaviours and, in particular, detect chaotic behaviour.

The structure of the paper is as follows: Section~\ref{sec:TFoSM} develops the theory of vector fields over discrete measure spaces and introduces a 2-parameter family of metrics to measure distance between these objects. This section is more mathematical and may be skipped by readers mostly interested in the data analysis methodology. Its conclusions are summarised in Figure~\ref{fig:pipeline} and the associated algorithm. The geometric methods developed in Section 2 are applied in Section~\ref{sec:applications} to two examples: solutions to the Ginzburg-Landau equation on a flat domain in section~\ref{sec:CGLE}, and solutions to the Gray-Scott equation on a sphere in Section~\ref{Turing}. The analysis of both these systems is compared to the Lyapunov exponents analysis in Section~\ref{sec:Lya}. Finally, Section~\ref{sec:conclusions} presents a conclusion to this work.

\section{The geometry of spatio-temporal data}

\label{sec:TFoSM}

In this section, we introduce the theory of vector fields over discrete measure spaces, the mathematical objects that will model complex spatio-temporal data such as time-dependent images, gradients of images, weather data on curved domains, neural activity on the brain, etc. We will show that data modelled as vector fields over discrete measure spaces can be endowed with a 2-parameter family of metrics. This is of particular interest in the context of dynamical systems since metrics can be used to distinguish time-dependent states, ultimately leading to an abstract representation of their dynamics. We will demonstrate that the use of MDS can allow us, among other things, to obtain low-dimensional Euclidean representations of complex systems' dynamics. 

We start this section by defining a discrete measure space and presenting some examples. Then, we define vector fields over discrete measure spaces and present examples of data that can be mathematically modelled using these objects. Next, we show that the set of vector spaces over a fixed discrete measure space can be endowed with a 2-parameter family of metrics to give rise to a family of metric spaces, the $\mathcal L^{p,q}$ spaces.  These spaces will allow us to track the time-evolution of a complex system. Finally, we present a pipeline to generate and analyse low-dimensional Euclidean representations of complex systems dynamics. 

\subsection{Discrete measure spaces}
Let $\mathcal M$ be a countable set and let $\mathcal P(\mathcal M)$ be its power set regarded as a $\sigma$-algebra. A measure $\mu$ on $\mathcal M$ is a function $\mu:\mathcal  P(\mathcal M)\to \mathbb R_{\geq0}\cup\{\infty\}$ satisfying the following conditions:
\begin{enumerate}
\item for any $A\in\mathcal P(\mathcal M)$, all $s\in A$ satisfy $\mu(\{s\})<\infty$;
\item for all $A\in \mathcal P(\mathcal M)$,  $\mu(A)=\underset{a\in A}\sum\mu(\{s\}).$
\end{enumerate}
Here $\{s\}$ denotes the singleton set corresponding to the element $s$. A discrete measure space $(\mathcal M,\mu)$ is a countable set $\mathcal M$ along with a measure $\mu$. If $\mathcal M$ is finite, we say that $(\mathcal M,\mu)$ is a finite measure space.

\subsubsection{Examples}

\begin{enumerate}
\item Let $S$ be a finite set. For all $s\in S$, we define $\mu(s)=n$, with $n$ a positive real. Then the pair $(S,\mu)$ is a discrete measure space.

\item Cubical complexes (Figure~\ref{fig:cub}).  
A $d$-cube of dimension $d$, $\omega$, is a product of finitely many elementary intervals with $d$ unitary intervals of the real line as factors. A set $X\subset \mathbb R^d$ is cubical if it is the union of finitely many elementary cubes, $X=\bigcup_{i=1}^r\omega_{i}$ satisfying certain glueing rules (see \cite{kaczynski2004}). A cubical complex of dimension $d$ is a cubical set $X\subset \mathbb R^m$, $m\geq d$, where the maximal dimension of its cubes is $d$. An $m\times n$-cubical lattice $\Lambda_{m\times n}$ of dimension $d$ is a cubical complex such that the centres of the  $d$-cubes $q$ lie over the points $x_0,\dots,x_{mn}$ which define a $m\times n$-lattice on $ \mathbb R^{mn}$.  
Let $S(\Lambda_{mn})$ denote the set of all elementary $d$-cubes $\omega$ of $\Lambda_{mn}$. The set $S(\Lambda_{mn})$  along with the measure $\mu$ given by $\mu(\omega)=1$ defines a discrete measure space.

\item Simplicial complex (Figure~\ref{fig:sim}). A simplex of dimension $n$ or $n$-simplex $\sigma$ is the collection of all points $x=\Sigma_{i=0}^nt_iv_i$, $0\leq i\leq n$, such that $\sum_{i=0}^n t_i=1$, $t_i\geq0$, for some geometrically independent set $V=\{v_0, v_1,\dots,v_n\}\subset \mathbb R^d$. The elements of $V$ are called the vertices of $\sigma$. 
The standard $n$-simplex $\Delta^n$ is defined by the set $V=\{e_0,e_1,\dots,e_n\}$, the canonical basis of $\mathbb R^{n+1}$. 

A simplicial complex $K$ is a finite collection of simplices such that whenever an $n$-simplex $\sigma$ belongs to $K$ so does any face ($k$-simplex, $k<n$) of $\sigma$, and if $\sigma,\tau\in K$ then $\sigma\cap\tau$ is either empty or a face of both. The dimension of $K$ is the maximum dimension of its simplices. Let $S(K)$ be the set of all simplices $\sigma$ of $K$ of maximum dimension. We define a measure $\mu$ on $S(K)$ by $\mu(\sigma)=\text{vol}(\sigma)$. The pair $(S(K),\mu)$ is a discrete measure space.

\item Voronoi diagrams (Figure~\ref{fig:vor}). Starting with a set $S$  of points, called seeds, on a metric space $X$, the Voronoi diagram of $S$, $\mathcal V(s)$ is a partition of $X$ into regions called cells, where each cell contains exactly one seed. These regions, called Voronoi cells, satisfy the condition that all points in a given region are closer to the seed of the region than to any other. Let $C\in \mathcal V(S)$ and $\text{vol}(C)$ be the volume of $C$. The pair $(\mathcal  V(S),\mu)$, where $\mu(C)=\text{vol}(C)$ defines a discrete measure space. 
\end{enumerate}

\begin{figure}
    ~\hfill
    \begin{subfigure}[t]{0.16\textwidth}
        \includegraphics[width=\textwidth]{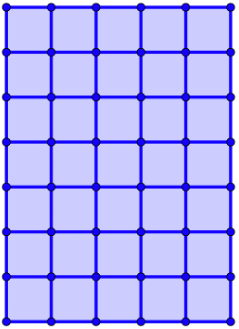}
        \caption{}\label{fig:cub}
    \end{subfigure}
    \hfill
    \begin{subfigure}[t]{0.23\textwidth}
        \includegraphics[width=\textwidth]{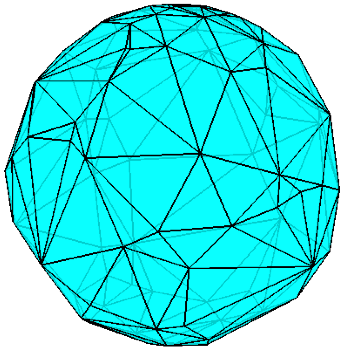}
        \caption{}\label{fig:sim}
    \end{subfigure}
    \hfill
    \begin{subfigure}[t]{0.23\textwidth}
        \includegraphics[width=\textwidth]{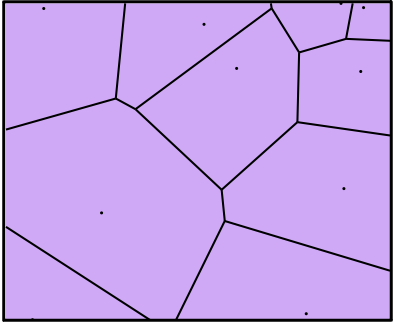}
        \caption{}\label{fig:vor}
    \end{subfigure}
    \hfill~
    \caption{Examples of discrete measure spaces: (a) cubical lattice; (b) simplicial complex; (c) Voronoi diagram on the plane.}
    \label{fig:dms}
\end{figure}

\subsection{Vector field representation of spatio-temporal data}
\label{sec:tensors_discrete}

Vector fields over discrete measure spaces might be seen as objects that generalise discrete vector fields over smooth manifolds. In general, we will not require the discrete measure spaces to be discretisations of smooth manifolds.  

Given a vector space $V$ of rank $d$ over a field $k$, let $e_1,\dots,e_d$ denote the elements of an orthonormal basis of $V$. Let $1\leq p,q\leq\infty$. For any vector $v=\sum_{i=1}^d a_ie_i$ in $V$, the $q$-norm $\|v\|_q$ of $v$ is given by 

\begin{equation}
\|v\|_{q}:=\bigg(\sum_{i=1}^d|a_{i}|^q\bigg)^{\frac{1}{q}},
\end{equation}
if $q<\infty$, and  by
\begin{equation}
\|v\|_\infty:=\text{max}\{|a_1|,|a_2|,\dots,|a_d|\},
\end{equation}
otherwise.
Given $v=\sum_{i=1}^d a_{i}e_{i}$ and $w=\sum_{i=1}^d b_i e_{i}$ in $V$, their inner product is defined as: 
\begin{equation}\label{eq:inner}
\langle v,w \rangle=\sum_{i=1}^d a_{i}b_{i}.
\end{equation}
From equation \eqref{eq:inner}, it follows that for $q=2$, the norm $\|v\|_2$ is induced by the inner product:
\begin{equation*}
\|v\|_2=\langle v,v\rangle^{\frac{1}{2}}.
\end{equation*}

Let $(\mathcal M,\mu)$ be a finite measure space and let $V$ be a vector space of dimension $d$ with a norm $\|-\|_q$. A \textit{discrete vector field of rank} $d$ \textit{over} $\mathcal M$ is a map $\mathbf X:\mathcal M\to V$, that sends a point $s\in\mathcal M$ to a vector  $\vfx{}(s)\in V $. Given a vector field of rank $d$ over $\mathcal M$, $\vfx{}$, we define the $L^{p,q}$-norm  $\|\vfx{}\|_{L^{p,q}}$ of $\vfx{}$ by
 \begin{equation}
\|\vfx{}\|_{L^{p,q}}:=\bigg( \sum_{s\in\mathcal M}  \|\vfx{}(s)\|^p_q \mu(s) \bigg)^{\frac{1}{p}}, \label{Lp_norm_discrete}
\end{equation}
if $p<\infty$, and by 
\begin{equation}
\|\vfx{}\|_{L^{\infty,q}}:=\max\{\|\vfx{}(s)\|_q: \mu(s)>0\},
\end{equation}
if $p=\infty$. 

\subsubsection{Examples}
  We present examples of data that can be regarded as vector fields over discrete measure spaces, such as RGB images or weather data (see Figure \ref{fig:dvf}).

\begin{enumerate}

\item Monochromatic and RGB images provide the first examples of data having the structure of a vector field over a discrete measure space, where the pixels define the discrete measure spaces and the light intensity at each pixel defines the value of the discrete vector field at a given element of the discrete measure space. Thus, in what follows, we define images in terms of the theory presented in the previous section.

A monochromatic image of width $w$ and height $h$ is a vector field of rank one $\mathcal X:S(\Lambda_{wh})\to \mathbb R$ over the discrete measure space $(S(\Lambda_{wh}),\mu)$, where $\mu(q)=1$, for all $q\in S(\Lambda_{wh})$. Similarly, an RGB image of width $w$ and height $h$ is a vector field of rank three $\mathcal X:S(\Lambda_{wh})\to \mathbb R^3$ over the discrete measure space $(S(\Lambda_{wh}),\mu)$. In each pixel, we associate the vector that corresponds to its RGB values.

\item Gradients of images. Let $\mathcal X:S(\Lambda_{w\times h})\to \mathbb R$ be a monochromatic image of width $w$ and height $h$. The image gradient $\nabla\mathcal X$ of $\mathcal X$ is
a vector field $\nabla\mathcal X:S(\Lambda_{w\times h}) \to \mathbb R^2$ whose value at the 3-dimensional pixel $q_{i,j}$ in the cubical lattice $\Lambda_{wh}$ is given by 
\begin{equation}
\nabla\mathcal X(q_{i,j})= \big(\mathcal X(q_{i+1,i})-\mathcal X(q_{i,i}),\mathcal X(q_{i,i+1})-\mathcal X(q_{i,i})\big).
\end{equation}
Similarly, given an RGB image $\mathcal X:S(\Lambda_{w\times h})\to \mathbb R^3$ of width $w$ and height $h$, the gradient of $\mathcal X$ is the vector field of rank six $J\mathcal X:S(\Lambda_{w\times h}) \to \mathbb M_{3\times2}\cong \mathbb R^6$,  defined by the matrix

\begin{equation}
J\mathcal X(q_{i,j}):=
\begin{pmatrix}
R(q_{i+1,i})- R(q_{i,i})& R(q_{i,i+1})-R(q_{i,i})\\
G(q_{i+1,i})-G(q_{i,i})&G(q_{i,i+1})-G(q_{i,i})\\
B(q_{i+1,i})-B(q_{i,i})&B(q_{i,i+1})-B(q_{i,i})
\end{pmatrix},
\end{equation}
where $R(q_{i,j})$,  $G(q_{i,j})$,  $B(q_{i,j})$, are the values of the red, green and blue channels, respectively, on the 3-dimensional pixel $q_{i,j}$. 

\item Given a $d$-dimensional simplicial complex $K$ with a discrete measure $\mu$, a vector field over $K$ is a function  $\mathcal X(K):S(K)\to V$, where $S(K)$ is the set of simplices of dimension $d$ of $K$ and $V$ is a vector space of dimension $d$. Let $\mu$ be the measure over the set $S(K)$ that to each element $\sigma\in S(K)$ assigns its volume vol$(\sigma)$. Data from numerical weather and climate models can be represented as a vector field over a simplicial complex. 
These models split the areas of interest on the surface of the Earth into a set of grids, giving rise to simplicial complexes, and then use observations of $d$ variables such as pressure, winds, temperature and humidity, etc.,  at different locations of the planet \cite{collins2013}.  In this case, the observations at each simplex in the Earth's simplicial model can be represented as a $d$-dimensional vector. 

\end{enumerate}

\begin{figure}
    \begin{subfigure}[t]{0.14\textwidth}
        \includegraphics[width=\textwidth]{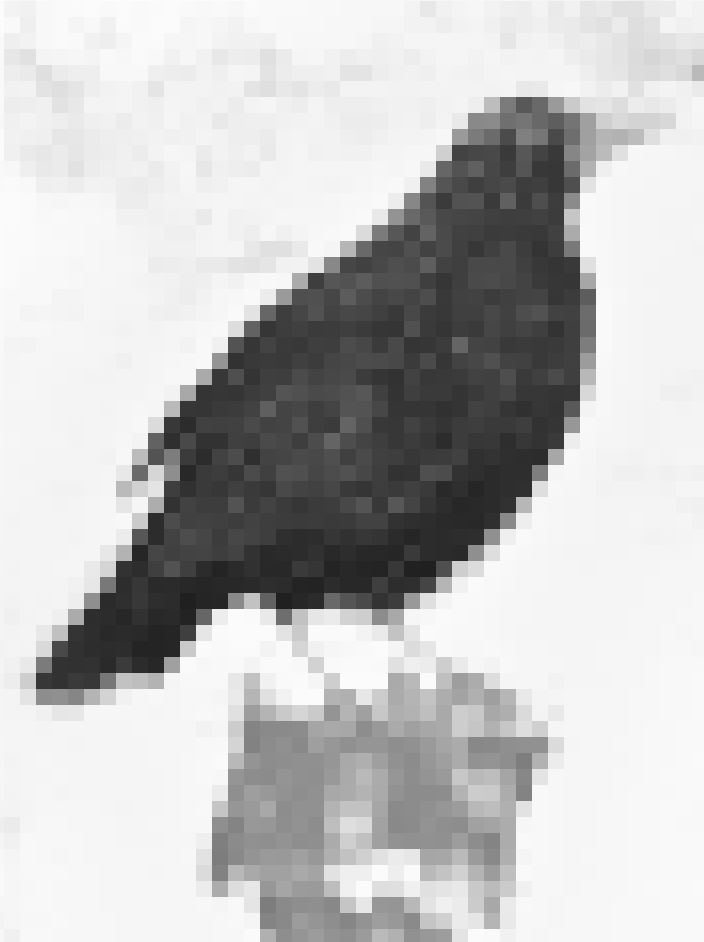}
        \caption{}
    \end{subfigure}
 \hfill
 \begin{subfigure}[t]{0.14\textwidth}
        \includegraphics[width=\textwidth]{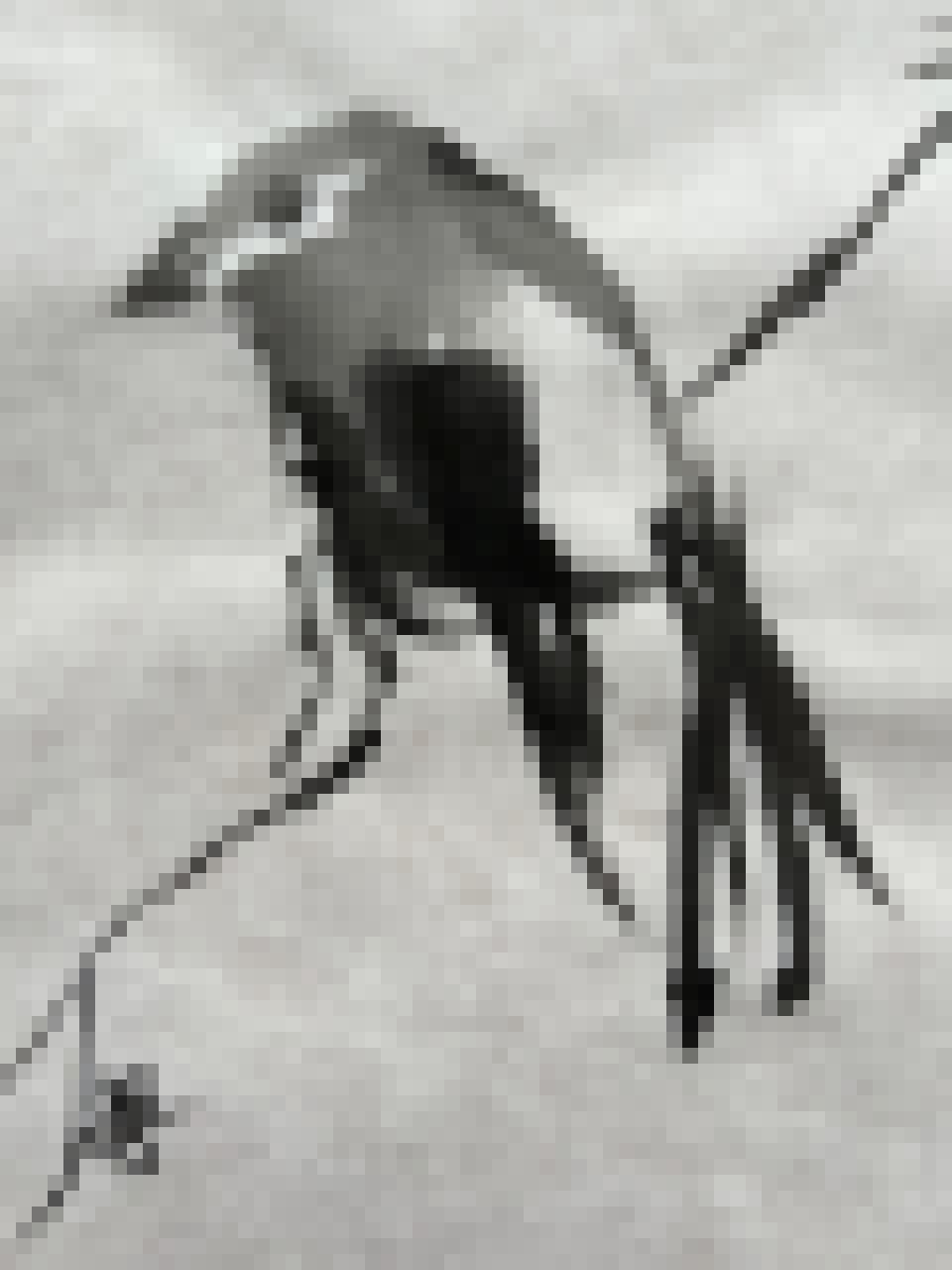}
        \caption{}
    \end{subfigure}
 \hfill
 \begin{subfigure}[t]{0.15\textwidth}
        \includegraphics[width=\textwidth]{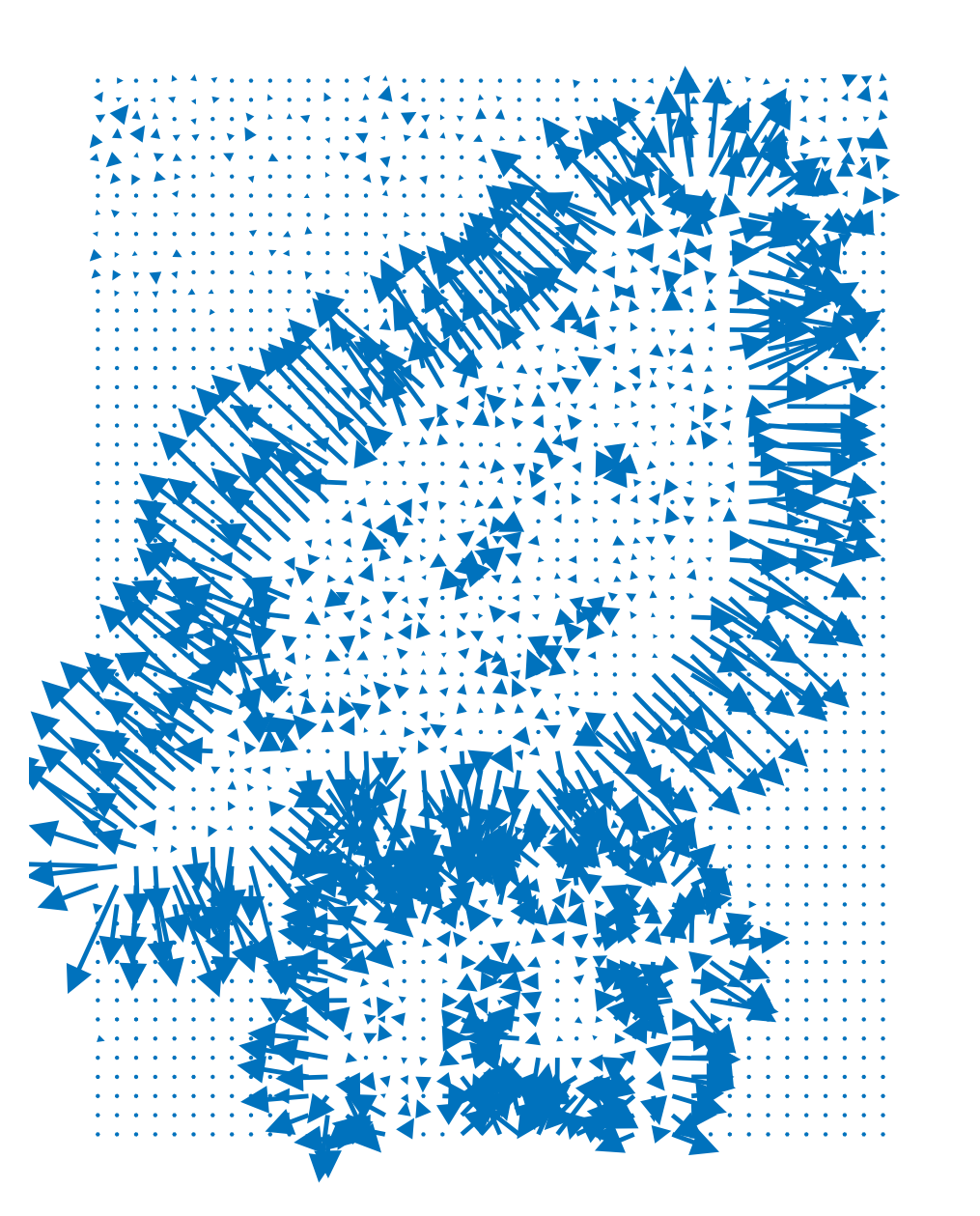}
        \caption{}
    \end{subfigure}
 \hfill
    \begin{subfigure}[t]{0.15\textwidth}
        \includegraphics[width=\textwidth]{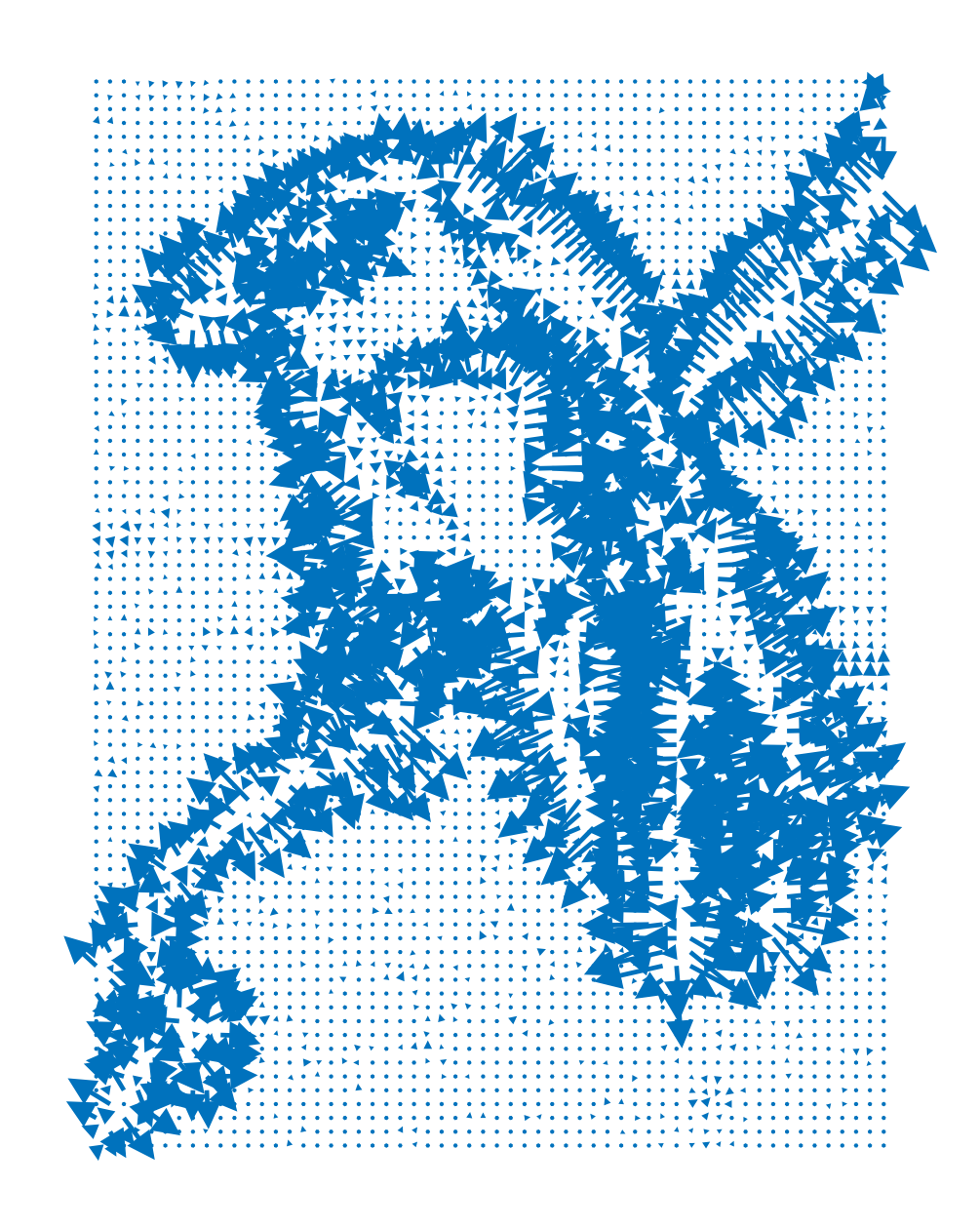}
        \caption{}
    \end{subfigure}
 \hfill
    \begin{subfigure}[t]{0.3\textwidth}
        \includegraphics[width=\textwidth]{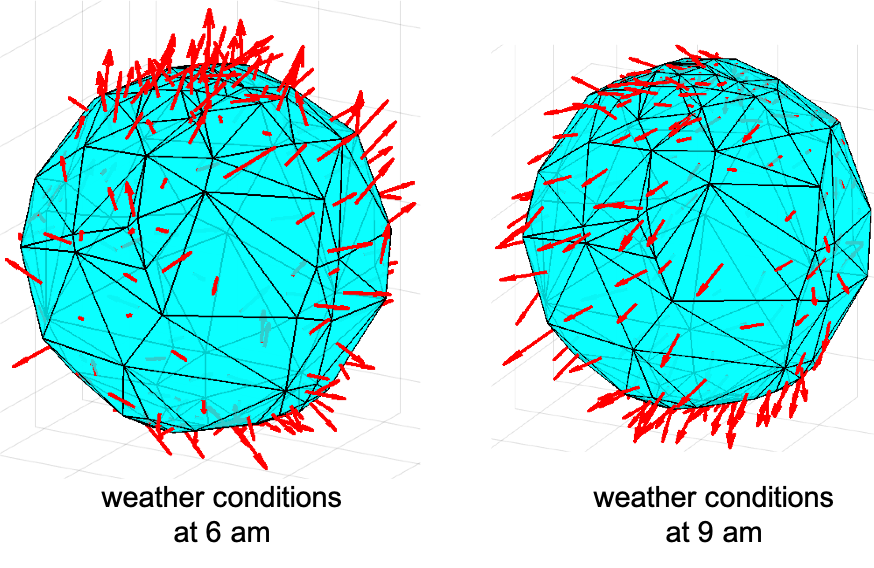}
        \caption{}
    \end{subfigure}
    \caption{Examples of vector fields over discrete measure spaces: (a,b) image, (c,d) image gradient and (e) weather data. The images in panels a to d are data sets on squared lattices, in which each pixel is an element of the lattice, while (e) are data sets on simplicial complexes.  Plots (c,d) are obtained by computing the gradient of the intensity of images (a) and (b), respectively and highlight regions of high contrast.}
    \label{fig:dvf}
\end{figure}

\subsection{$\mathcal L^{p,q}$ spaces} 

We now endow the set of all vector fields of rank $d$ over a fixed discrete measure space with a 2-parameter family of metrics. Metrics allow us, in some sense, to measure how similar two vector fields of rank $d$ over a discrete measure space are. Consequently, we can use these metrics to measure similarities between images, gradients of images, functions on simplicial complexes, or any other type of data that can be represented as a vector field over a discrete measure space. In the context of complex systems, we can use these metrics to distinguish different temporal states of a given complex system. All the pairwise distance values between system states will contain information about the system dynamics.

\label{sec:LpMetMeasure}
For $p,q\in [1,\infty]$ and $(\mathcal M,\mu)$ a discrete measure space, we defined the $L^{p,q}$-metric or distance function on the set of discrete vector fields of rank $d$ over $\mathcal M$ as follows: given two vector fields of rank $d$ over $(\mathcal M,\mu)$, $\vfxx$ and $\vfyy$, their distance, denoted $d_{L^{p,q}}$, is defined as
\begin{equation}
d_{L^{p,q}}(\vfxx,\vfyy) :=\| \mathbf{X}-\mathbf{Y} \|_{ L^{p,q}}\label{Lp_dist_discrete}
\end{equation}
For $p=q=2$, we can express the $L^{2,2}$-norm and the $L^{2,2}$-metric in terms of the inner product in $V$:

\[\|\vfx{}\|_{ L^{2,2}}=\bigg( \sum_{a\in\mathcal M}  \langle \vfx{}(s),\vfxx(s)\rangle \mu(s) \bigg)^{\frac{1}{2}},\]

\[d_{ L^{2,2}}(\vfxx,\vfyy) := \bigg( \sum_{a\in\mathcal M}  \langle \vfww(s),\vfww(s) \rangle\mu(s)  \bigg)^{\frac{1}{2}},\]
where $\vfww(s)=\vfx{}(s)-\vfy{}(s)$.
\subsubsection{Examples}
\begin{enumerate}
    \item  We can use the $L^{p,q}$- metrics to define distances between either monochromatic or RGB images. Given two RGB images $\mathcal X:S(\Lambda_{wh})\to \mathbb R^3$ and $\mathcal Y:S(\Lambda_{wh})\to \mathbb R^3$, their $ L^{p,q}$-distance is given by
\begin{equation}\label{eq:distim}
 {\mathcal L^{p,q}}(\mathcal{X},\mathcal{Y}) = \bigg( \sum_{s\in S(\Lambda_{wh})}  \|\mathcal {X}(s)-\mathcal{Y}(s)\|_{p}^{q} \mu(s) \bigg)^{\frac{1}{p}}.
 \end{equation}
where $\mathcal X(s)=(R,G,B)$ and $\mathcal Y(s)=(R',G',B ')$ are the RGB values at the pixel $s$ and $\mu(s)=1$. Observe that  for the monochromatic case, the distance is also given by expression \eqref{eq:distim} where $\mathcal X{}(s)$ and $\mathcal Y{}(s)$ are real numbers. 

\item Given two $RGB$ images $\mathcal X$ and $\mathcal Y$, the distance between their image gradients $J\mathcal X$ and $J\mathcal Y$ is given as

\begin{equation}\label{eq:distimgrad}
 d_{L^{p,q}}(J\mathcal X,
 J\mathcal Y) = \bigg( \sum_{s\in S(\Lambda_{wh})}  \|J\mathcal X{}(s)-J\mathcal Y{}(s)\|_q^p\mu(s)  \bigg)^{\frac{1}{p}}.
 \end{equation}

\item Let be $K$ be a simplicial complex and $\mu=\text{vol}(\sigma)$. Given two vector fields of rank $d$, $\mathcal X(K)$ and $\mathcal Y(K)$, over $(K,\mu)$, their $L^p$-distance is defined as

\begin{equation}
 d_{L^{p,q}}(\mathcal X(K),\mathcal Y(K)) = \bigg( \sum_{\sigma\in S(K)}  \|\mathcal X{}(\sigma)-\mathcal Y{}(\sigma)\|_q^p{\rm{vol}}(\sigma)  \bigg)^{\frac{1}{p}}.
 \end{equation}

\end{enumerate}
The following result shows that the spaces of vector fields of rank $d$ endowed with the $L^{p,q}$ metric are both vector and metric spaces. Further, this result shows that the $L^{2,2}$ metric is indeed induced by an inner product. 

\begin{theorem} 
Let $(\mathcal M,\mu)$ be a discrete measure space, $V$ be  a vector space of rank $d$ and $1\leq p,q\in \mathbb R$
\begin{enumerate}
\item The collection of all discrete  vector fields of rank $d$ over $\mathcal M$ along with the $L^{p,q}$-metric define a metric space that we denote $\mathcal L^{p,q}(\mathcal M,\mu,V)$.
\item The space $\mathcal L^{p,q}(\mathcal M,\mu,V)$ is a vector space.
\item If p=q=2, then $\mathcal L^{p,q}(\mathcal M,\mu,V)$ is a vector space with an inner product.
\end{enumerate}
\end{theorem}

\subsection{Low-dimensional Euclidean representations of $\mathcal L^{p,q}$ spaces} 
\label{sec:DoToDMS}

 In the previous subsection, we endowed the space of vector fields of rank $d$ over a fixed discrete measure space with a family of metrics. These metrics allow us to measure similarity between data points for spatio-temporal data with the structure of a vector field over a discrete measure space. In particular, for time-evolving complex systems, a metric measures how different the states of a system are at different points in time. One can obtain a matrix of distances from pairwise distances of the system states. In general,  distance matrices of non-Euclidean metrics are not easy to interpret, so a Euclidean representation or embedding of the dynamics of the system is often needed. To this end, we can make use of MDS \cite{cox2008}. This unsupervised learning algorithm permits obtaining embeddings of metric spaces into Euclidean spaces. If the initial matrix of pairwise distances is Euclidean, the embedding will be isometric. Otherwise, the algorithm will produce an embedding with minimal distortion in the distances. Here, we briefly describe how the embedding of a matrix of $ L^{p,q}$ distances is obtained using MDS.

 Let $S=\{F_i\}_{i=1}^n$ be a set containing $n$ data points modelled as vector fields of rank $d$ over a discrete measure space $\mathcal M$ with measure $\mu$. Let $D^{(2)}$ be the matrix of squared distances of the elements of $S$ obtained using the $L^{p,q}$ metric, $1\leq p,q\in\mathbb R$. The matrix of embedding coordinates $A$ associated with the matrix $D^{(2)}$ is obtained from the eigenvalue decomposition of the matrix given by $B=-\frac{1}{2}CD^{(2)}C$. Here $C=I_n-\frac{1}{n}J_n$, where $I_n$ is the $n\times n$ identity matrix and  $J_n$ is the $n\times n$ matrix of all ones. Let $Y$ be the matrix of eigenvectors and $\Lambda$ be the diagonal matrix of the square roots of the eigenvalues of $B$. Then the matrix of embedding coordinates of  $S$ is given by $A=Y\Lambda$.  
 
The Euclidean representation of data obtained from studying complex systems may be high-dimensional. However, the number of degrees of freedom with significant dynamics may be much smaller. It is, therefore, important to identify these principal degrees of freedom, as they contain significant information about the dynamics. The eigenvalues of the matrix $B$ will allow us to identify the principal degrees of freedom of the system. The number of principal degrees of freedom or coordinates of the system is equal to the number of eigenvalues of the matrix $B$ that capture most of the variability of the data, e.g. 80\% or more. We will use these embedding coordinates to visualise the dynamics of the systems in low dimensions.

Next, we will use the matrix of embedding coordinates to approximate each state of the complex system as a linear combination of the first $k$ principal coordinates. 

\begin{theorem}[Low-rank approximation via principal coordinales]\label{teo:low-rank-approximation}
Let $F_1, \ldots, F_n \in \mathcal{L}^{p,q}(\mathcal{M}, \mu, V)$ and let
\[
\overline{F} = \frac{1}{n} \sum_{i=1}^n F_i
\]
be their centroid. Let $A \in \mathbb{R}^{n \times d}$ be the matrix of embedding coordinates such that the $i$-th row $A_i = a(F_i) \in \mathbb{R}^d$ is the embedding of $F_i$, and suppose the linear map
\[
\theta: \mathbb{R}^d \to \mathcal{L}^{p,q}(\mathcal{M}, \mu, V)
\]
satisfies $\theta(a(F_i)) = F_i - \overline{F}$, and is represented by a matrix $H \in \mathbb{R}^{d \times m}$ with respect to fixed standard bases. Let $F \in \mathbb{R}^{n \times m}$ be the matrix whose rows are the coordinates of $F_i - \overline{F}$. Then:

\begin{enumerate}
    \item The matrix $F$ satisfies
    \[
    F = A H^T.
    \]

    \item If $A^T A = \Lambda^2$ for a diagonal matrix $\Lambda$ of positive entries, then
    \[
    H^T = \Lambda^{-2} A^T F.
    \]

    \item For any $k \leq d$, the approximation of $F_i$ using only the first $k$ principal coordinates is given by
    \[
    F_i^{(k)} = \overline{F} + \sum_{j=1}^k \sum_{s=1}^n A_{ij} \Lambda_{jj}^{-2} A_{sj} (F_s - \overline{F}).
    \]
\end{enumerate}
\end{theorem}

\begin{proof}
Let $F_1,\ldots,F_n$ be the original data points represented as vector fields of rank $d$, and let $A$ be the matrix of their embedding coordinates. Then the $i$-th row of the matrix $A$ is the embedding coordinates $a(F_i)$ of $F_i$ in $\mathbb{R}^d$.  Let $\overline{F}$ be the centroid of the points $F_i$, i.e.\ $\overline{F} = \frac{1}{n} \sum_{i=1}^n F_i$.  Suppose that the map sending $a(F_i)$ to $F_i-\overline{F}$ extends to a linear map $\theta\colon \mathbb{R}^d \to \mathcal L^{p,q}(\mathcal M,\mu,V)$, and let $H$ be the matrix representing this map with respect to the standard bases for $\mathbb{R}^d$ and $\mathcal L^{p,q}(\mathcal M,\mu,V)$. Here, in order to define a standard basis on $\mathcal L^{p,q}(\mathcal M,\mu,V)$, we pick an arbitrary total order on $S$, which we fix from now.  If $F$ is the matrix whose $i$-th row consists of the coordinates of $F_i-\overline{F}$ with respect to the standard basis for $\mathcal L^{p,q}(\mathcal M,\mu,V)$, it then follows that $F = A H^T$.  Note that $A^T A = \Lambda^T Y^T Y \Lambda = \Lambda^T \Lambda = \Lambda^2$ since $Y$ is orthogonal, therefore $H^T = \Lambda^{-2} (A^T A H^T) = \Lambda^{-2} A^T F$.  In particular, the $j$-th basis vector of $\mathbb{R}^d$ corresponds to the $j$-th row of $H^T = \Lambda^{-2} A^T F$.

Now in order to approximate $F_i$ using the first $k \leq d$ principal coordinates, $F_i^{(k)}$, we use the approximation $F_i^{(k)} = \overline{F} + \sum_{j=1}^k A_{ij} h_j$, where $h_j$ is the $j$-th row of $H^T$, so that the $t$-th coordinate of $F_i^{(k)}$ is
\[
(F_i^{(k)})_t = \overline{F}_t + \sum_{j=1}^k A_{ij} (\Lambda^{-2} A^T F)_{jt} = \overline{F}_t + \sum_{j=1}^k \sum_{s=1}^n A_{ij} \Lambda_{jj}^{-2} A_{sj} F_{st},
\]
where $\overline{F}_t$ is the $t$-th coordinate of $\overline{F}$.  In particular, we have
\[
F_i^{(k)} = \overline{F} + \sum_{j=1}^k \sum_{s=1}^n A_{ij} \Lambda_{jj}^{-2} A_{sj} (F_s-\overline{F})\].
\end{proof}

Theorem \ref{teo:low-rank-approximation} generalises the idea behind PCA-based reconstruction to the setting of vector-valued functions over a discrete measure space. It highlights how the structure of the original $\mathcal L^{p.q}$ space is preserved and approximated via the embedding and how one can recover approximations from the principal coordinates. Thus, it is important to point out that the low-dimensional approximation described in this section is analogous to those obtained via classical PCA for the cases when the vector field is a function defined over a rectangular domain, such as monochromatic images or scalar functions over rectangular meshes. Moreover, in cases when the rectangular domain is too large, classical PCA becomes computationally expensive as the dimension of the data space becomes larger. For more general vector fields, such as vector-valued functions over simplicial complexes, classical PCA cannot be used as a method for dimensionality reduction. 

\subsection{Pattern recognition in low-dimensional representations of  $\mathcal L^{p,q}$ spaces}
Figure \ref{fig:pipeline} shows a pipeline for the application of the theory of spaces of vector fields over discrete measure spaces to the geometric analysis of data collected from the study of dynamical systems. In this example pipeline, the vector fields of rank three over discrete domains are RGB images generated as colour maps of a real periodic function over a rectangular grid. Choosing $1\leq p,q\in \mathbb R$, we can compute all the $L^{p,q}$ distances between images to produce a distance matrix. We use the MDS algorithm to produce an Euclidean embedding. We plot the eigenvalues to determine the number of principal coordinates or degrees of freedom of the dynamics. We plot then the principal coordinates of the system, that is, the projection of the embedding onto its principal coordinates. We use Theorem \ref{teo:low-rank-approximation} to generate approximations of the RGB images using only the $k$ coordinates that capture the highest variance of the data, that is, the coordinates with the largest eigenvalues.

\begin{figure}[t]
    \centering
    \includegraphics[width=1\textwidth]{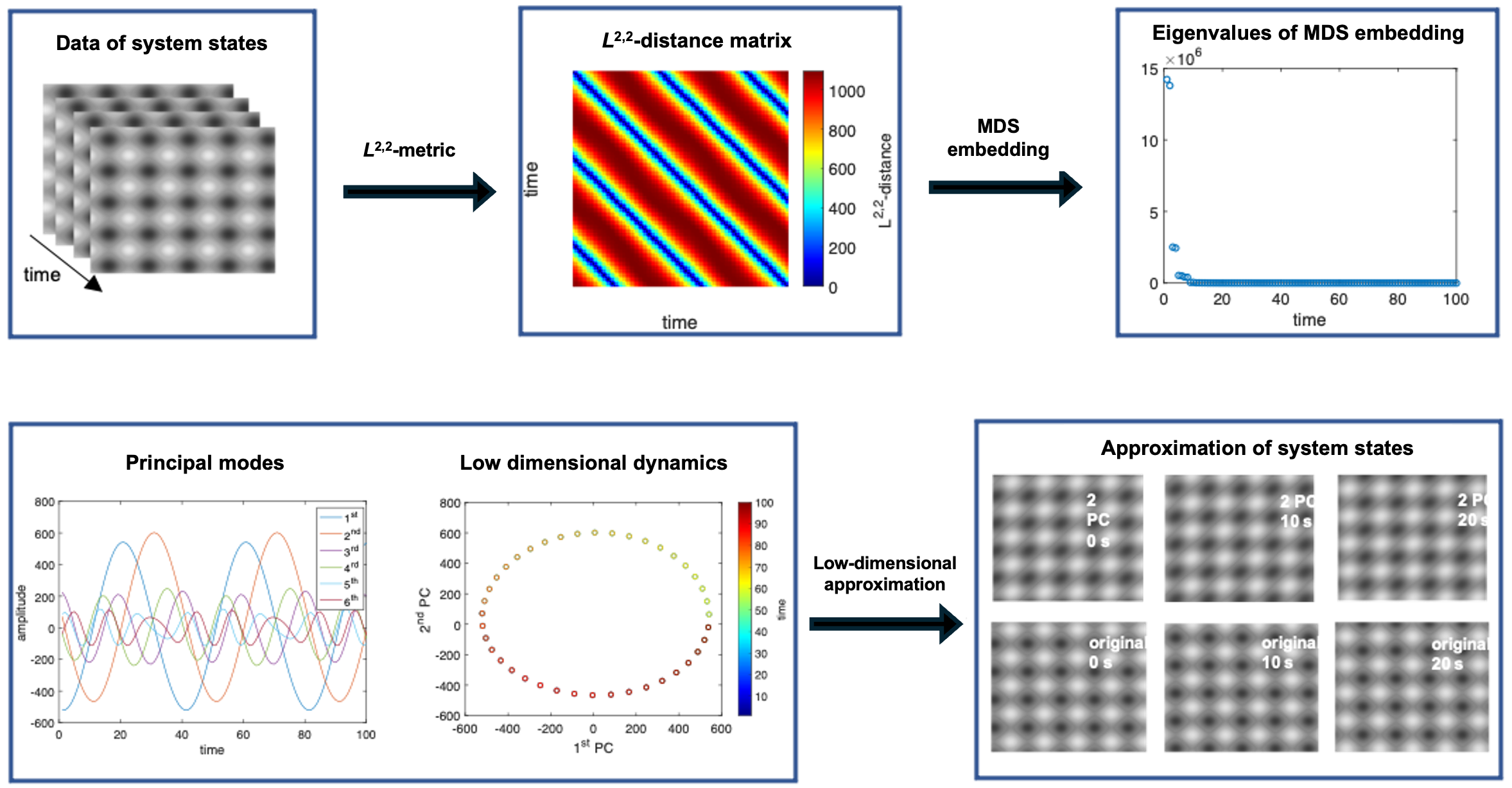}
    \fbox{\parbox[t]{0.8\textwidth}{\small
    \textbf{Algorithm}
    \begin{enumerate}
        \item  Store time-dependent data as vector fields of rank $d$ on a discrete measure space (e.g. RGB images) [left corner upper row].
        \item Choose an $L^{p,q}$ metric; compute distances between data points to generate a matrix that contains all pairwise distances between time-dependent data points [centre upper row].
        \item Use the MDS algorithm to obtain a matrix of embedding coordinates and associated spectrum.
        \item Analyse the spectrum to detect the $k$ principal coordinates; plot their amplitudes either as functions of time or as a low-dimensional phase space [left panel bottom row].
        \item Use the principal coordinates to reconstruct the data [right panel, lower row].
    \end{enumerate}
    }}
    \caption{Pipeline for the application of the spaces of  vector fields of rank $d$ to complex systems dynamics.} 
    \label{fig:pipeline}
\end{figure}

\section{Applications to complex systems dynamics} \label{sec:applications}
In this section, we use the theory of vector fields over discrete measure spaces, described in the previous section, to study synthetic dynamical systems from a data-driven approach. In particular, we use this method in the detection of global attractors and chaotic behaviour. As examples of systems that exhibit chaotic behaviour, we apply the theory described in Section~2 to solutions of the Ginzburg-Landau equation~\cite{Chate1996} and the Gray-Scott model~\cite{Pearson1993}. These systems were chosen because they are well studied and have interesting applications, and because their solutions have a broad range of behaviour, incorporating various levels of order or chaos. For instance, the Ginzburg-Landau equation is widely used as a model system to study spatio-temporal chaos~\cite{Chate1996}, which is a very common state in spatially extended dynamical systems. It is applied to many natural phenomena, such as Rayleigh-B\'{e}nard convection, superconductivity, and predator-prey relationships~\cite{Mocenni2010,Sherratt2009}. The Gray-Scott model was chosen as an example because in certain parameter regimes the solutions do not equilibrate to a homogeneous state but form Turing patterns, 
and in some turbulent regimes they can exhibit oscillating behaviour. Moreover, these types of systems have been used as morphogenesis models for brains and other biological tissues~\cite{lefevre2010,kelso1995,kondo2010,turing1990}. The Gray-Scott model is a convenient choice for a Turing pattern generator because its behaviour is well-studied~\cite{Pearson1993} and it has chaotic states with interesting dynamics, as discussed in Section~\ref{Turing}.

\subsection{Ginzburg-Landau equation on a flat domain} \label{sec:CGLE}

As a first proof of concept, we study image data of different dynamical states generated by the two-dimensional complex Ginzburg-Landau equation (CGLE). One is a stable \lq frozen' state; the two other states, which are chaotic, are termed \lq defect turbulence' and \lq defect turbulence with spirals'~\cite{Chate1996}. The variety of states makes the CGLE suitable for analysis of spatiotemporal chaos. This application of the method outlined in Section~\ref{sec:TFoSM} serves as an example of detecting spatio-temporal chaos in images, or other data on a flat domain.

\subsubsection{Data generation}
\label{GLDG}
Image data was generated by integrating the CGLE in a flat domain with periodic boundary conditions. The CGLE is given by: 
\begin{equation}
    \frac{\partial A}{\partial t} = A + (1+i\alpha)\boldsymbol{\nabla}^2A - (\beta-i)|A|^2 A, \label{CGLE}
\end{equation}

where $A$ is a complex scalar field and $\alpha$ and $\beta$ are real parameters. Integrating the CGLE yields different dynamical behaviours for different parameter choices; the chosen parameter values used in this paper are $\alpha=2$, $\beta=5$ for the frozen state, $\alpha=2$, $\beta=1$ for defect turbulence, and $\alpha=0$, $\beta=0.56$ for defect turbulence with spirals.

Equation~\eqref{CGLE} was solved in Python using the split-step method on a square grid of $128 \times 128$ pixels with periodic boundary conditions. For the nonlinear steps, a third-order explicit Runge-Kutta method was used. The results were verified through visual comparison to the results by Chat\'{e} and Manneville~\cite{Chate1996}. 
Furthermore, numerical convergence of the results was verified for spatial and temporal resolution. The numerical solutions can be regarded as vector fields of rank one over a regular rectangular lattice. In this case, a measure over the lattice is defined by assigning a measure of value one to each pixel.
A vector field of rank two over a regular rectangular lattice was created from the scalar field $A$ by taking the image gradient. That is, we associated the corresponding gradient vector with each pixel of the image. We used the gradient function in Python to compute these vectors.  We applied the theory outlined in Section~\ref{sec:TFoSM} to the gradient vector field in addition to the scalar field for two reasons: on the one hand, the gradient field is of interest whenever one needs to analyse the variation of the corresponding scalar field. 
On the other hand, a common form of data is images in RGB format, which are often converted to greyscale for analysis. However, they can be represented as fields of vectors with 3 components instead without the loss of information that comes with greyscale conversion. 
Gradient vector fields serve as an example that the analysis presented in this paper applies to RBG images and other data with a vector field type. 

\subsubsection{Time dynamics analysis and phase space reconstruction}
Using the $L^{1,1}$ and $L^{2,2}$ metrics defined on the set of vector fields of rank one over a $128\times128$ mesh, the distance matrices corresponding to the solutions of the Ginzburg-Landau equation were obtained for the three types of different states. The $L^{1,1}$-distance matrices have larger values than the corresponding matrices of the $L^{2,2}$-distance,  which can also be appreciated in their low-dimensional MDS embeddings (Figure 4). A similar behaviour is observed in the case of the distance matrices and embeddings for the image gradients of the solutions of the Ginzburg-Landau equation (Figure 5). 

\begin{figure}[t]
    \centering
    \begin{subfigure}[t]{0.45\textwidth}
        \includegraphics[width=\textwidth]{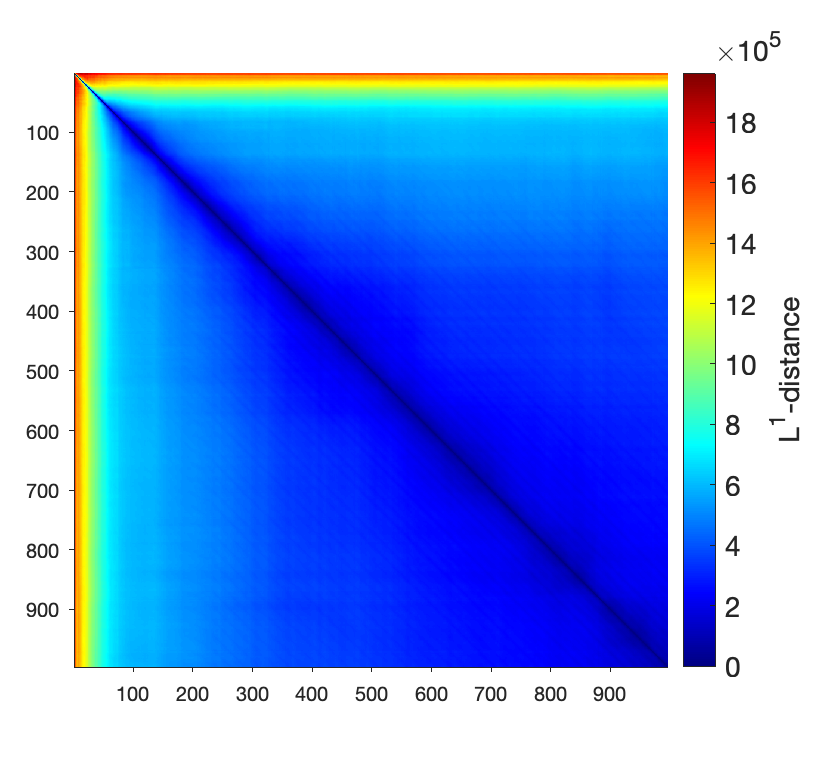}
        \caption{}
    \end{subfigure}
    \begin{subfigure}[t]{0.45\textwidth}
        \includegraphics[width=\textwidth]{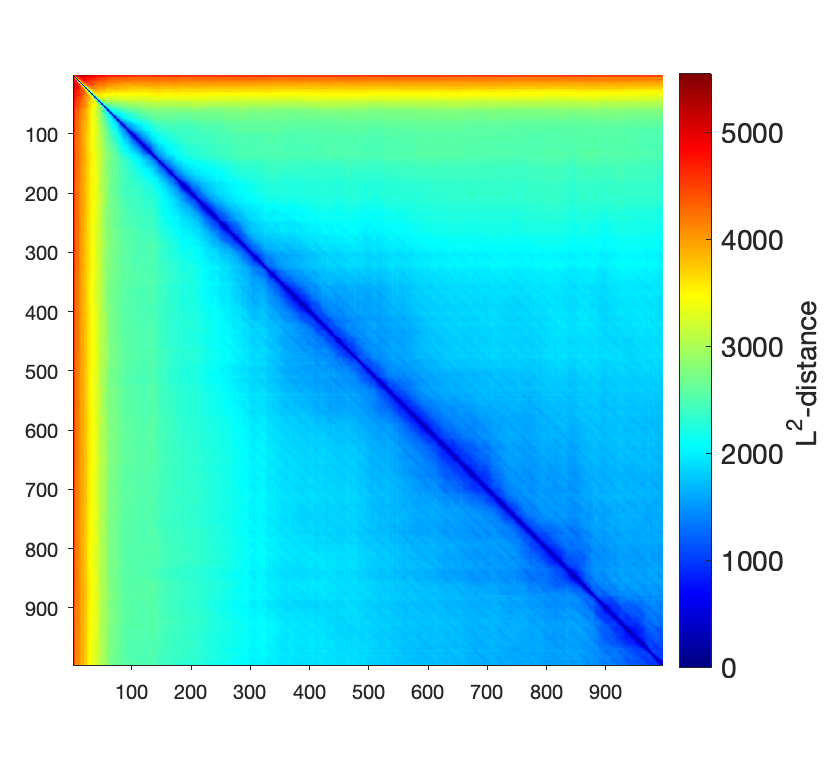}
        \caption{}
    \end{subfigure}
    \\
    \begin{subfigure}[b]{0.45\textwidth}
        \includegraphics[width=\textwidth]{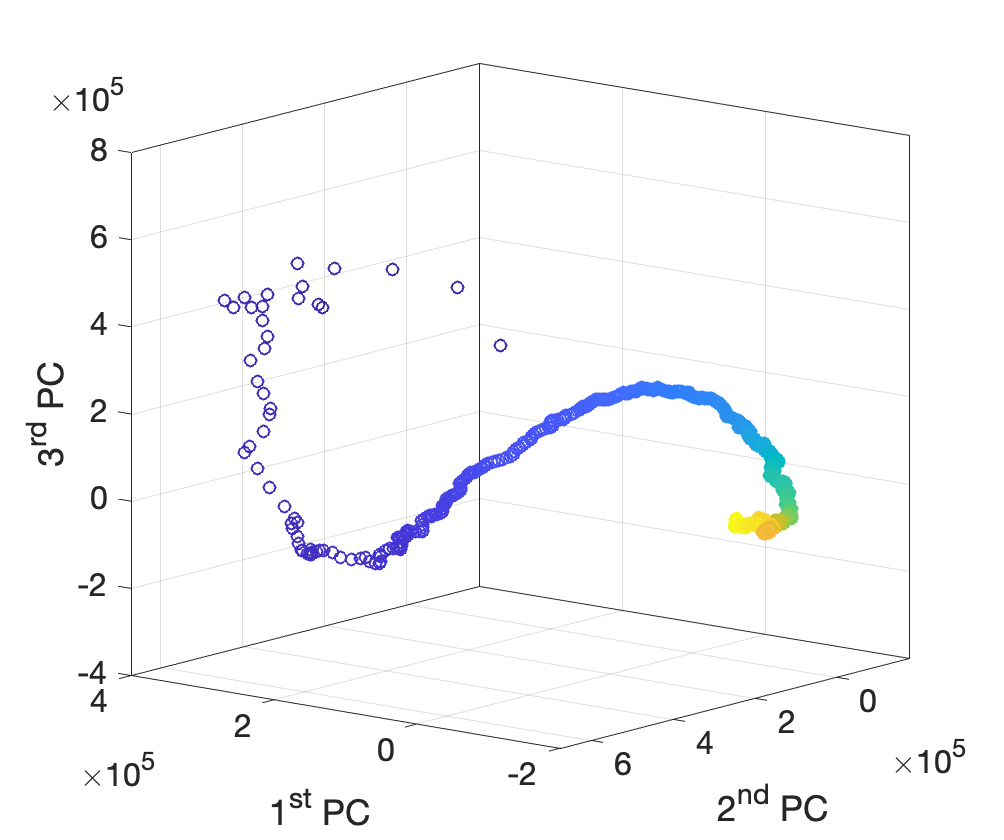}
        \caption{}
    \end{subfigure}
    \begin{subfigure}[b]{0.45\textwidth}
        \includegraphics[width=\textwidth]{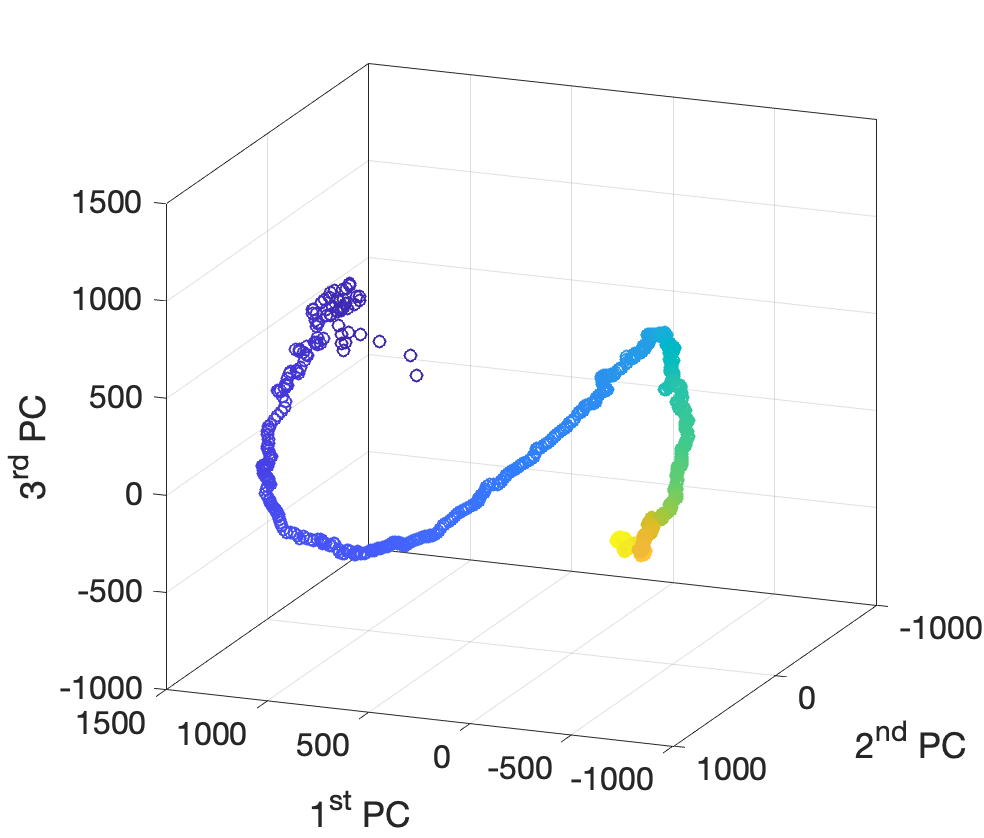}
        \caption{}
    \end{subfigure}
    \caption{Time-evolution of the frozen states of the Ginzburg-Landau equation.(a,b)  Distance matrices using the $L^p$-metric for $p=1,2$. The corresponding 3-dimensional embedding for the cases (c) $p=1$ and (d) $p=2$.    }
    \label{fig:Lp-frozen}
\end{figure}

\begin{figure}[t]
    \centering
    \begin{subfigure}[t]{0.45\textwidth}
        \includegraphics[width=\textwidth]{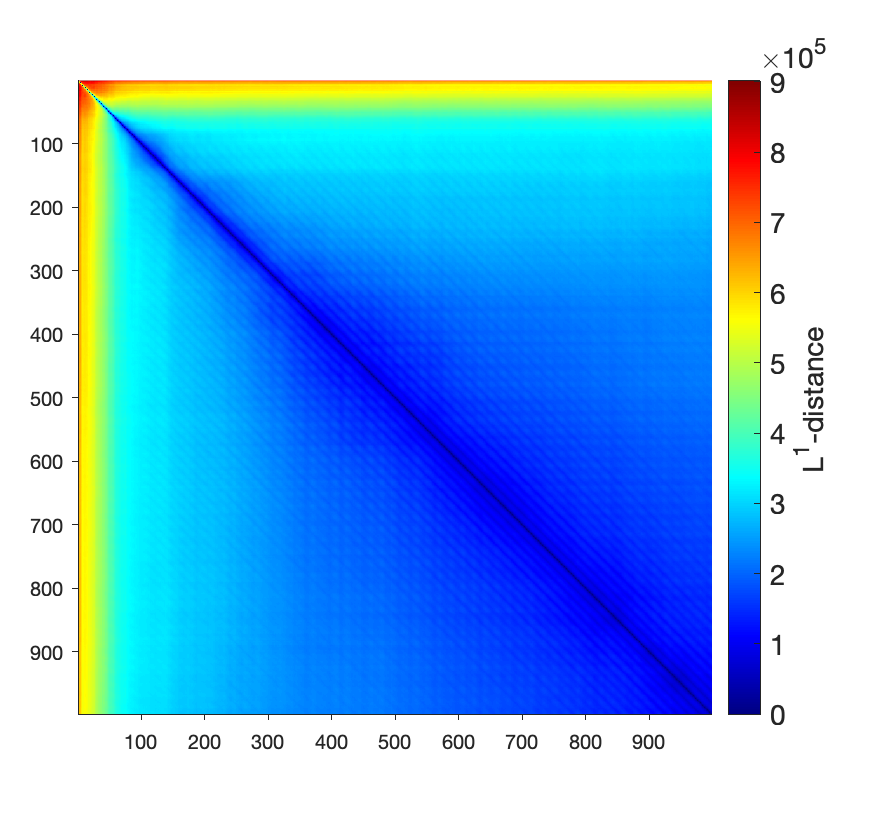}
        \caption{}
    \end{subfigure}
    \begin{subfigure}[t]{0.45\textwidth}
        \includegraphics[width=\textwidth]{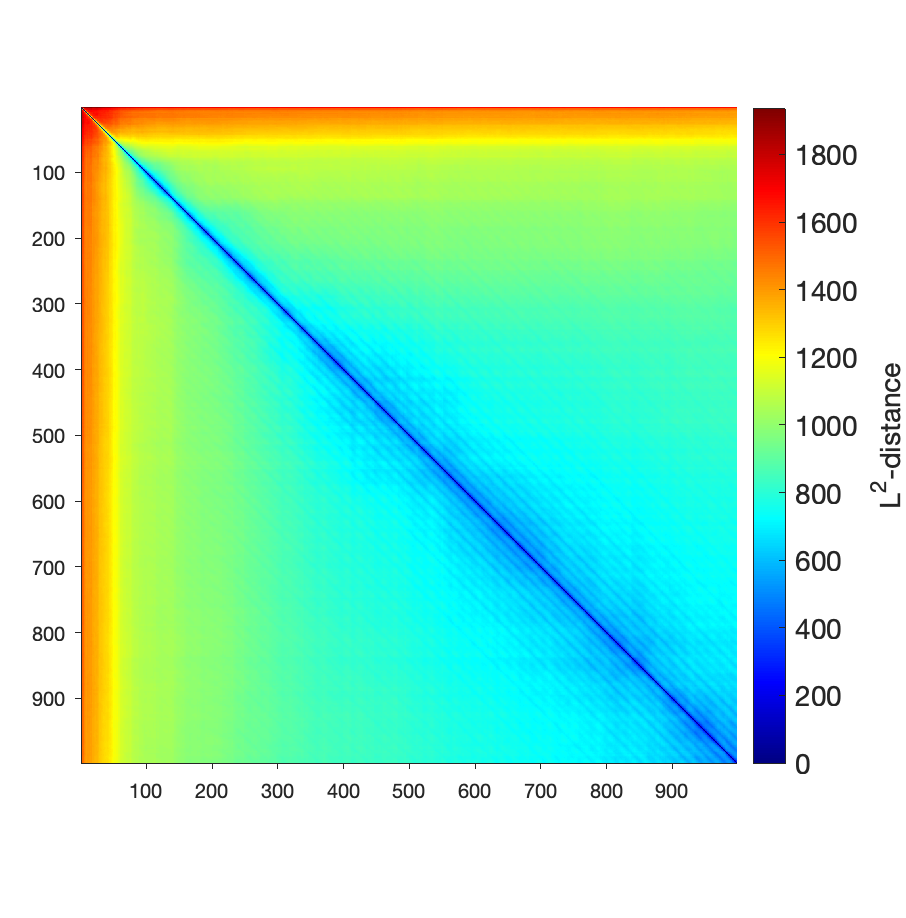}
        \caption{}
    \end{subfigure}
    \\
    \begin{subfigure}[b]{0.45\textwidth}
        \includegraphics[width=\textwidth]{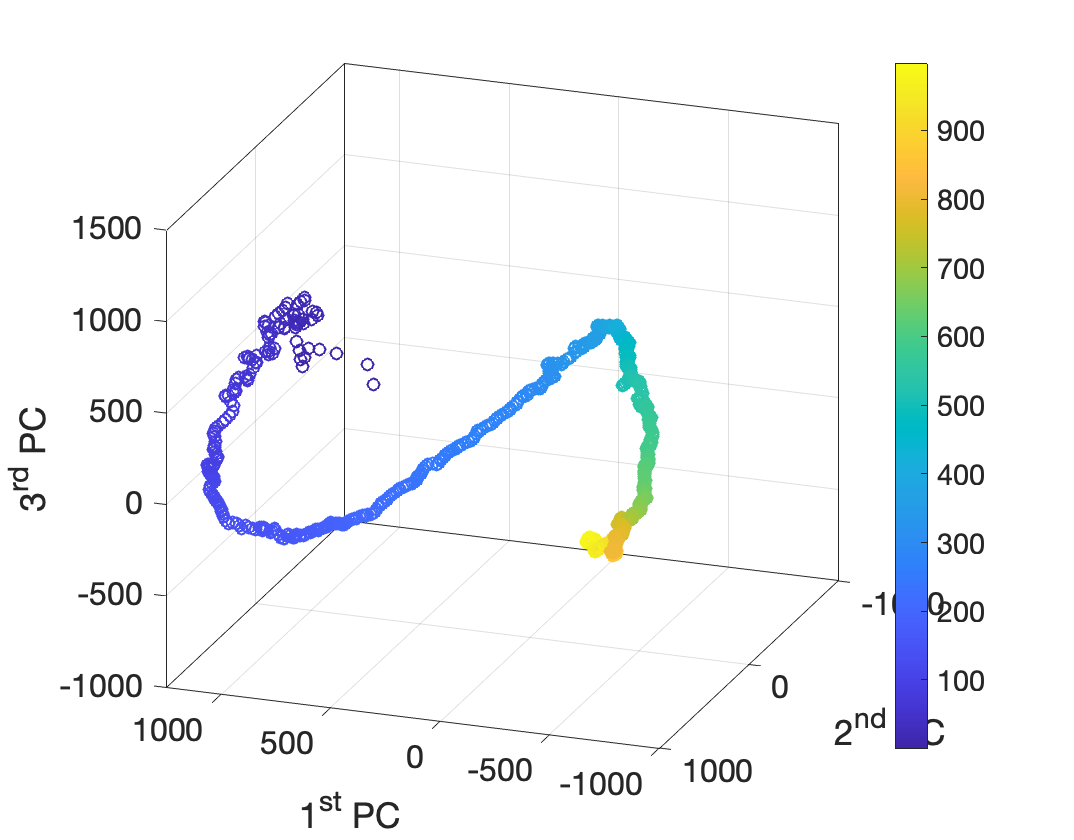}
        \caption{}
    \end{subfigure}
    \begin{subfigure}[b]{0.45\textwidth}
        \includegraphics[width=\textwidth]{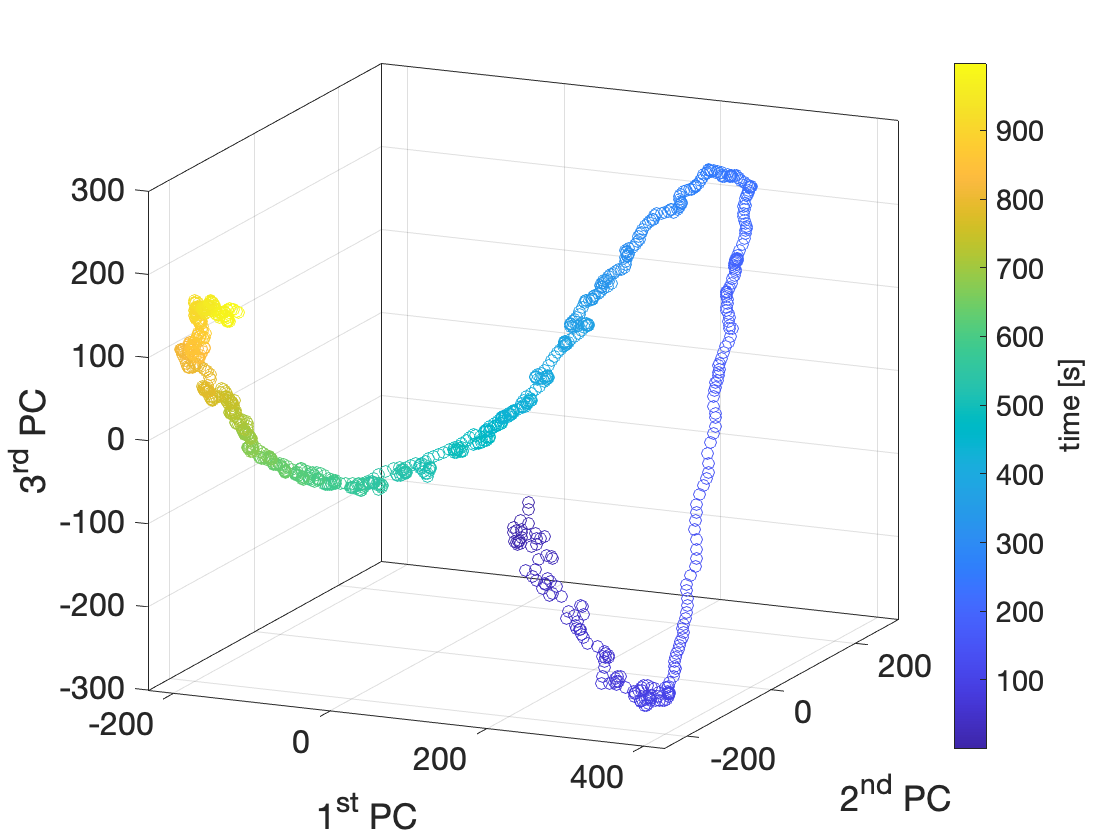}
        \caption{}
    \end{subfigure}
    \caption{Time-evolution of the gradient field of frozen states of the Ginzburg-Landau equation.(a) $L^{1,1}$ and (b) $L^{2,2}$ distance matrices for the evolution of the gradient field of frozen states. Embeddings of the (c) $L^{1,1}$ and (d) $L^{2,2}$ distance metrics via classical multidimensional scaling }
    \label{fig:Lp-frozen}
\end{figure}

\begin{figure}[t]
    \centering
    \begin{subfigure}[t]{0.45\textwidth}
        \includegraphics[width=\textwidth]{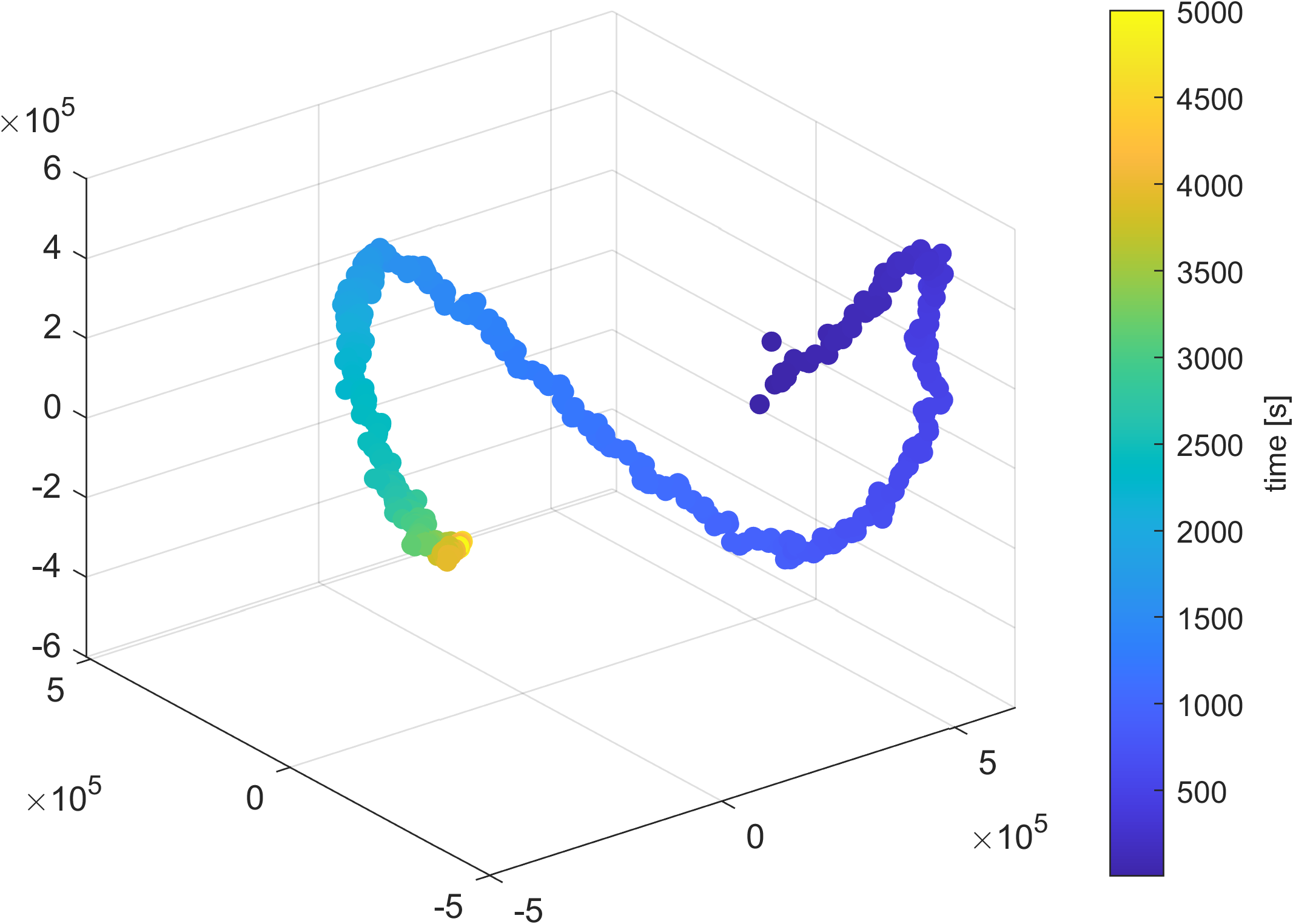}
        \caption{}
    \end{subfigure}
    \begin{subfigure}[t]{0.45\textwidth}
        \includegraphics[width=\textwidth]{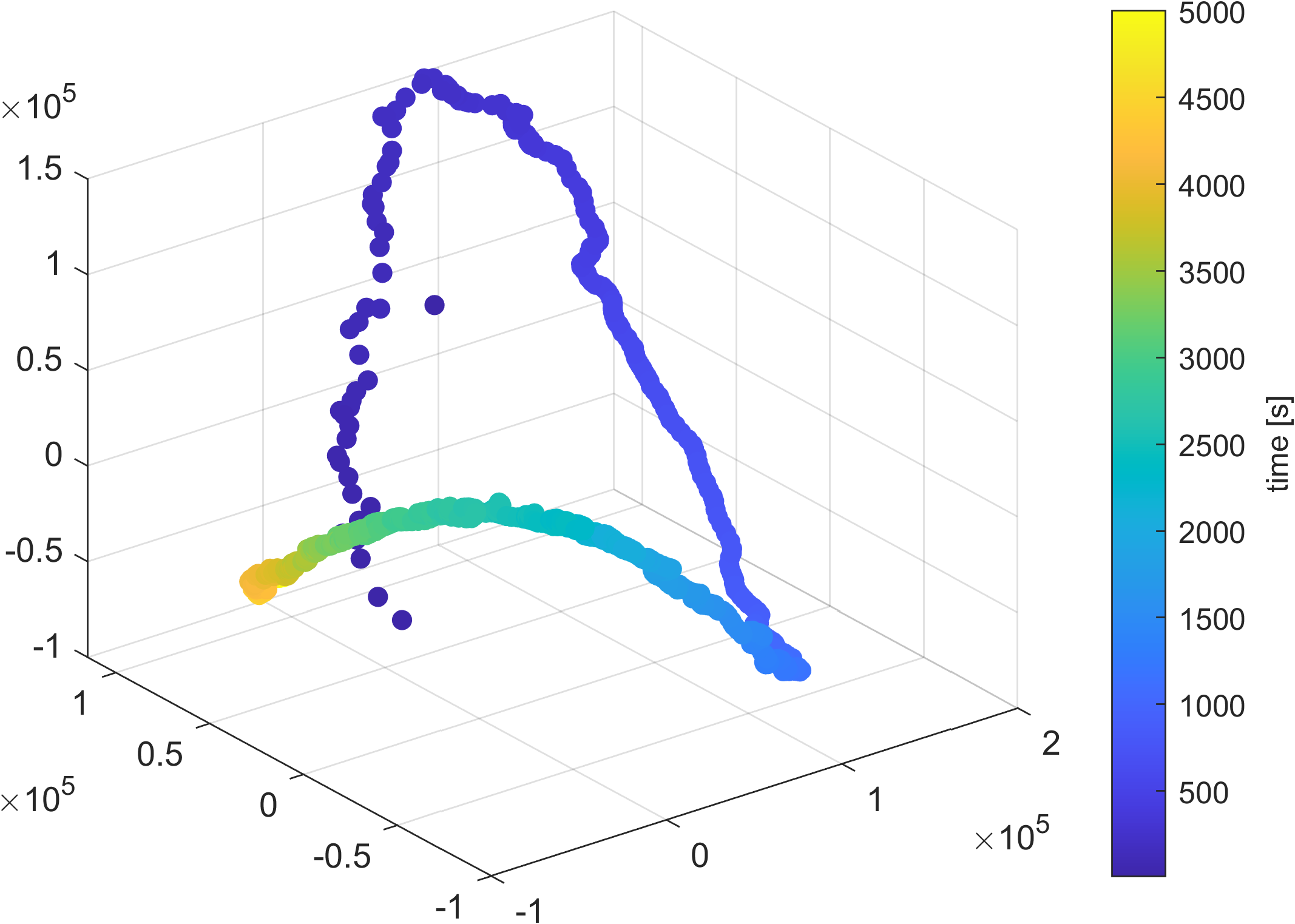}
        \caption{}
    \end{subfigure}
    \\
    \begin{subfigure}[b]{0.45\textwidth}
        \includegraphics[width=\textwidth]{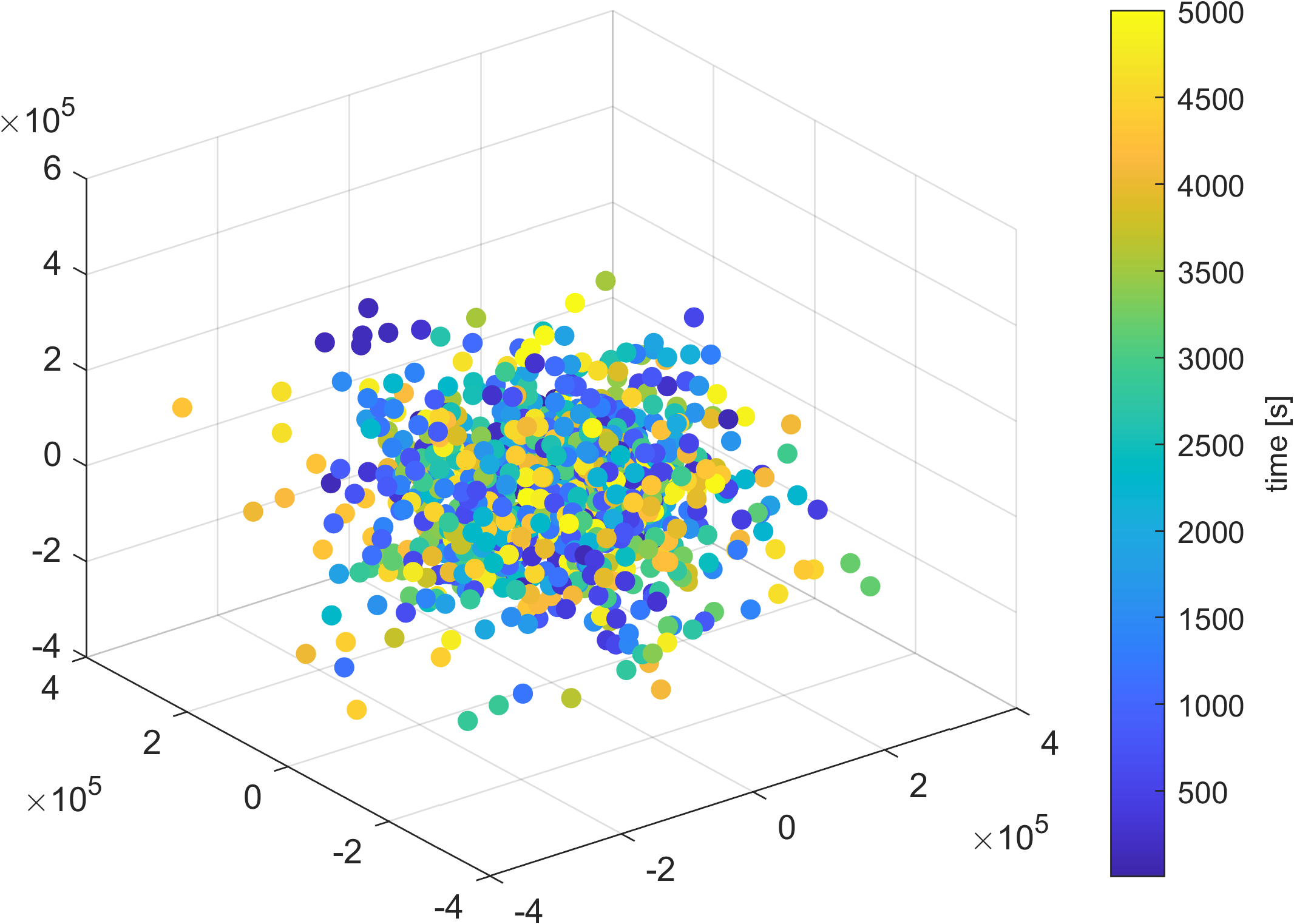}
        \caption{}
    \end{subfigure}
    \begin{subfigure}[b]{0.45\textwidth}
        \includegraphics[width=\textwidth]{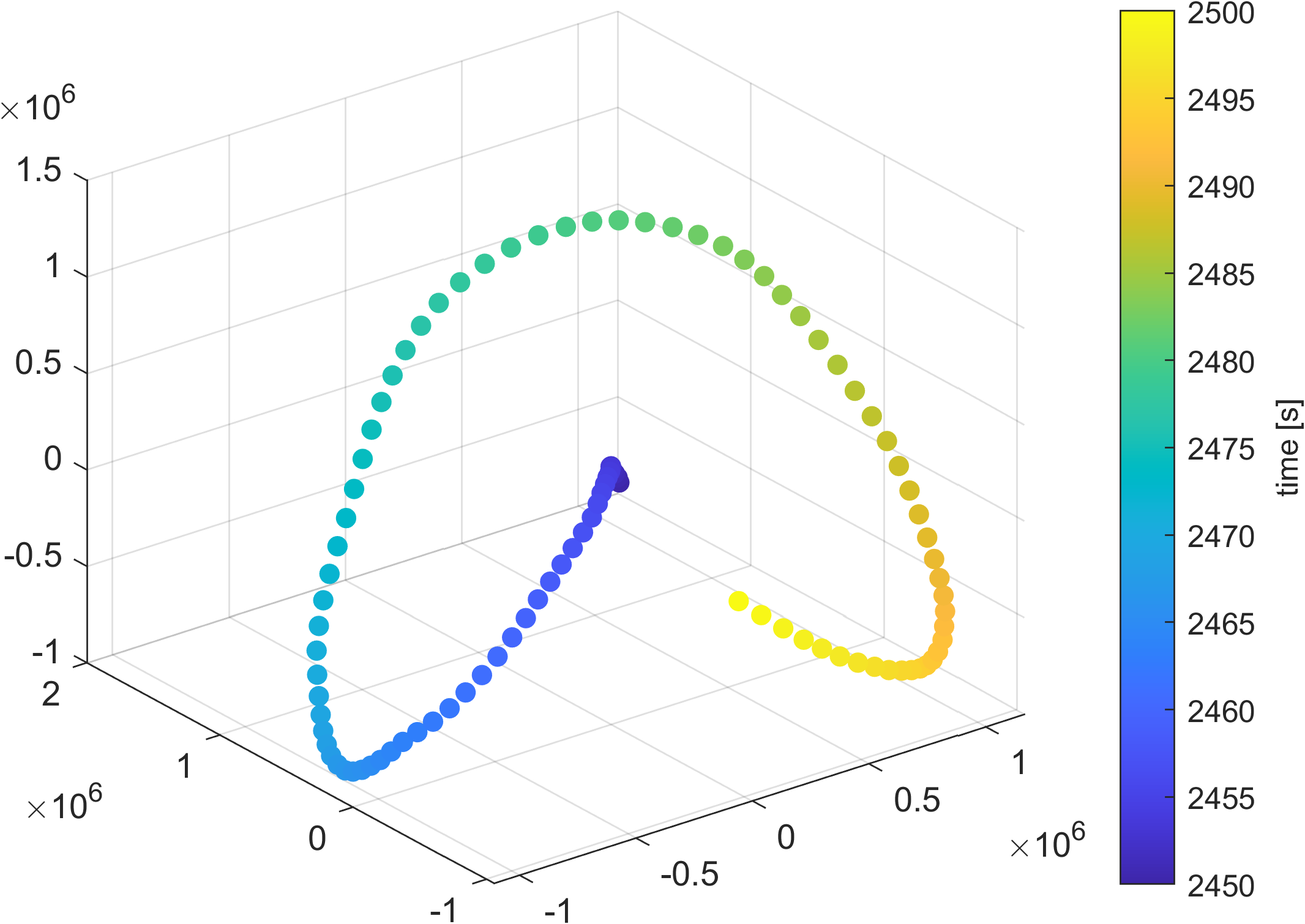}
        \caption{}
    \end{subfigure}
    \caption{Embedding of the gradient vector field $L^{2,2}$distance matrix for different phases: (a) defect turbulence with spirals, (b) frozen states and (c) defect turbulence. The time range is 5000s with 5s increments. (d) shows defect turbulence for a higher time resolution of 0.5s, between 2450 and 2500s.}
    \label{fig:dist_long3}
\end{figure}

 In Figure~\ref{fig:dist_long3}, the low-dimensional MDS embeddings of the $L^{2,2}$-distance matrices of the image gradients of the three attractors are shown for comparison.  
 The data was simulated over a length of 5000 time units, to capture the characteristic dynamics of each state over small and large time scales. 
 
The distance matrix embeddings show some trajectories, where in the case of a stable configuration or stationary point, the distances become small over time. For the chaotic states, the MDS embedding shows no trajectory at small timescales. Only for higher temporal resolution do trajectories become distinguishable. This absence of embedding trajectories for chaotic systems is found for both the scalar and gradient vector fields. At a time-stepping resolution of 5 time unit increments, the defect turbulence with spirals shows a trajectory similar to that of the frozen states, while defect turbulence does not. However, this state does show clear trajectories at 0.5 time unit resolution. Of note is that the spiral defect turbulence state develops frozen domains over long time scales. This likely explains why this state exhibits dynamical order at a lower resolution than the other defect turbulence state. Again, both the scalar and the gradient vector field analysis exhibits these results.

In Figure~\ref{fig:GLmoddec} we summarise the image analysis of the time dynamics of three attractors of the Ginzburg-Landau equation. We have used the $L^{2,2}$-metric on the images associated to the scalar field $A$.  Columns of Figure~\ref{fig:GLmoddec} correspond to three attractors discussed in Section~\ref{GLDG} in order of increasing dynamical complexity, namely frozen states, turbulence with spirals and defect turbulence from left to right. The rows from top to bottom represent a snapshot frame of the dynamics, the distance matrix between the frames, the eigenvalues, the amplitude of the four most significant principal coordinates and the embedding of the dynamics in a three-dimensional space, with the colour coding of the orbit indicating time evolution.   In the case of the frozen states and the defect turbulence with spirals, the dynamics become quiescent after an initial transient as a metastable equilibrium is reached.  A few eigenmodes dominate the eigenvalue spectrum, and the embedded orbit is a relatively simple trajectory that reaches a fixed point for sufficiently long times (yellow colour in figure~\ref{fig:GLmoddec}a4 and figure~\ref{fig:GLmoddec}b4).   The scenario is completely different in the case of defect turbulence (column c in figure~\ref{fig:GLmoddec}).  The eigenvalue spectrum, \ref{fig:GLmoddec}c2, is continuous, the dynamics of the principal components, \ref{fig:GLmoddec}c3, is irregular and the embedded orbit, \ref{fig:GLmoddec}c4, has no structure, indicating that the attractor has a number of dymensions much larger than 3. Interestingly, though, an embedding for a shorter time, inset of figure~\ref{fig:GLmoddec}c4, shows that this turbulent dynamics is composed of a regular orbit and it may be possible to use this to study the local structure of the attractor. 

\begin{figure}[t]
    \centering
    \includegraphics[width=1\textwidth]{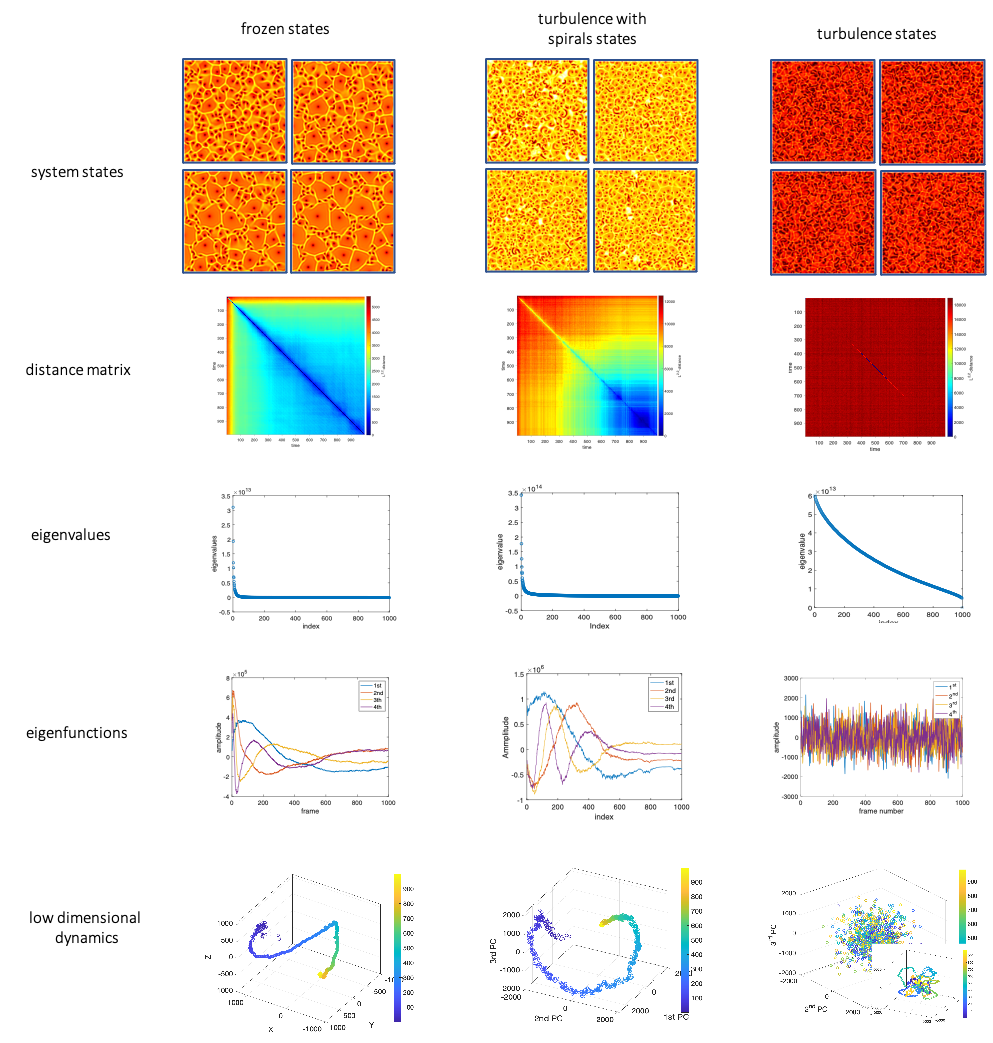}
    \caption{MDS embedding and low-dimensional representation of complex systems dynamics. The left column shows the results for frozen states; the middle column shows the results for the turbulence with spiral states, and the right column shows the results for the chaotic states. The gap between eigenvalues for the frozen and turbulent with spirals states indicates that the dynamics can be approximated in low dimensions; this is in contrast with the behaviour of the totally turbulent system, where dimensionally reduction is not applicable; moreover for this system, for short timesteps, it is possible to observe some trajectories of the system (bottom right corner).}
    \label{fig:GLmoddec}
\end{figure}

\subsubsection{Dimensionality reduction} \label{sec:GLdecomp}
We decompose the image data according to the method outlined in section~\ref{sec:DoToDMS}, with the aim of reducing the large number of dimensions while keeping the main data properties. What we observe is that expressing a system state as a linear combination of its first two principal coordinates recovers some of the main visible features of the system state. This confirms the suitability of using principal eigenvalues and principal coordinates to produce a good approximation of the system dynamics in low dimensions. 

The eigenvalues of the multidimensional scaling of the Ginzburg-Landau solutions 
are shown in the third row of Figure~\ref{fig:GLmoddec}. The results show that only the first three eigenvalues are significant in the case of the frozen states and turbulence states with spirals. This means that it is possible that the dynamics of these systems in low dimensions can be understood using the principal coordinates with the highest eigenvalues. 
Therefore, we approximate the states of the system, given as greyscale images, using only the first three principal components obtained as explained in Section \ref{sec:DoToDMS}. The approximated solutions using only three principal components are shown in Figure~\ref{fig:beta} and retain the main shape of the exact solution, validating the use of this method to generate low-dimensional approximations of the system dynamics. The dynamical behaviour of the turbulence system is different. There is no small subset of eigenvalues that retains most of the variance of the data. This result suggests that the dynamical attractor is high-dimensional. Therefore, it is not possible to find a low-dimensional representation for the system in this case. 

\begin{figure}
    \centering
    \includegraphics[width=0.9\textwidth]{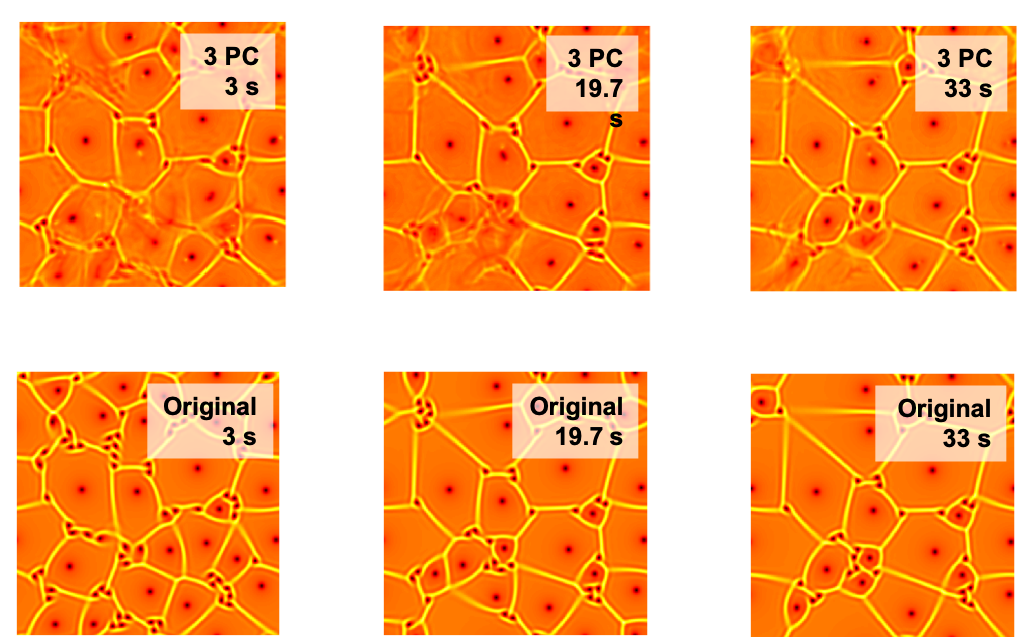}
    \caption{Mode approximation of the frozen states of the Ginzburg-Landau equation. The top row shows the approximation of the frozen states at 3, 19.7 and 33 seconds using 3 principal coordinates. In the bottom row, we show the non-approximated frozen states for comparison.}
    \label{fig:beta}
\end{figure}

\subsection{Turing patterns on a curved domain} \label{Turing}

We consider reaction-diffusion systems (Turing pattern formation) on a triangulated sphere for the second example. To the best of our knowledge, an extension to non-flat surfaces has not been previously attempted for similar analytical methods. We analyse states that form stable Turing patterns as well as states which exhibit spatio-temporal chaos. We aim to show that our method can detect the presence of a chaotic attractor on a curved surface, when this surface is modelled as a discrete measure space.

\subsubsection{Data generation}
The data was generated using the Gray-Scott model~\cite{Pearson1993}:
\begin{subequations}
    \begin{align}
        \frac{\partial u}{\partial t} &= D_u\nabla^2u + uv^2 - (A+B)u,\\
       \frac{\partial v}{\partial t} &= D_v\nabla^2v - uv^2 + A(1-v),
    \end{align}
\end{subequations}
where $D_u$ and $D_v$ are diffusion coefficients and $A$ and $B$ are reaction coefficients. For the stable Turing pattern generation, the parameters were set as $D_u=10^5$, $D_v = 2\cdot10^5$, $A=0.04$. The value of $B$ was varied to produce different patterns. The turbulent solution with oscillations was generated with the previously mentioned parameter values, except now $A=0.01$ and $B=0.033$ (see~\cite{Pearson1993} for an overview of the Gray-Scott solutions and appropriate parameter choices). The equations were solved in Comsol Multiphysics 5.6 on a spherical surface with radius $2.5/2\pi$, using the general form boundary PDE interface. A discrete domain was created by producing a triangular mesh on the sphere, consisting of 6670 triangular elements. The time range over which the model was solved was $2\cdot10^4$ time units 
with 20 time unit increments between outputs. The initial condition was taken to be $u=0$ and $v=1$, except in a region of 90 by 90 degrees, where initially $u=0.25$ and $v=0.5$. A random factor between 0 and 1 was added to these initial conditions to generate noise. The Comsol model output 
was imported in Matlab, where the area of all mesh elements was computed. The value of the field $u$ and its gradient at each element was taken to be the average of the values at the nodes of the element. A vector field created from the scalar field by taking the image gradient.  

\subsubsection{Time dynamics analysis and phase space reconstruction}
With the solution obtained at each mesh element, the $L^{2,2}$-distance was calculated using equation~\eqref{Lp_norm_discrete}, and the distance between fields at different times is given by equation~\eqref{Lp_dist_discrete}. 
We present the results of the geometric pipeline applied to the three Turing systems in Figure~\ref{fig:turing_all}. The eigenvalues obtained from MDS embedding show that the dynamics of all three systems can be approximated in low dimension using the first three principal coordinates. The eigenvalue analysis allows us to reduce the dynamics to three dimensions, since most of the data's variance is captured along the directions spanned by the first three principal coordinates. 
The 3-dimensional embeddings of the systems dynamics show, qualitatively, the different stages of the system dynamics for the Turing pattern states: 1) the diffusion of the noise in the initial condition; one long red stripe is formed against a homogeneous background; 2) the transformation of this stripe into the Turing pattern; 3) a steady state of the patterns with slight movement, which results in a mostly straight line in the embedding plots. The analysis of the eigenfunctions show differences in the behaviour of the three systems. The first thing to notice is that the eigenfunctions of the stripes and dots patterns seem to converge at large times. In contrast, the eigenfunctions show an oscillatory behaviour for the chaotic state. For the formation of the dot patterns, the system first has to generate stripes, which later break into pieces, leading to the formation of dots. Even though the eigenfunctions converge for the dots and stripes formation, their shape is more intricate for the case of the dot patterns. All these observations show that the eigenfunctions also capture the complexity of the dynamical systems. 

The embedding of the distance matrices of the solutions allows us to detect dynamical orbits for the chaotic system, as shown at the bottom right. Large timescale embeddings show the system's turbulence, while some trace of the oscillations is still visible; there is a faint pattern in the color distribution. 
These examples demonstrate the applicability of our method to qualitatively characterise the global dynamics of a system on non-flat domains. 

\subsubsection{Dimensionality reduction}

We apply the vector field decomposition method to approximate the solutions of the Gray-Scott system data using the first principal coordinates.
In Figure~\ref{fig:turing_approx} we present the approximation of the system states of the striped Turing pattern using only the first three principal components. We can see that this approximation preserves the system's principal features while reducing the data's dimensions. These results confirm that the suitability of the method to provide a low-dimensional approximation of complex system dynamics using non-flat domains. 

\begin{figure}[t]
    \centering
    \includegraphics[width=1\textwidth]{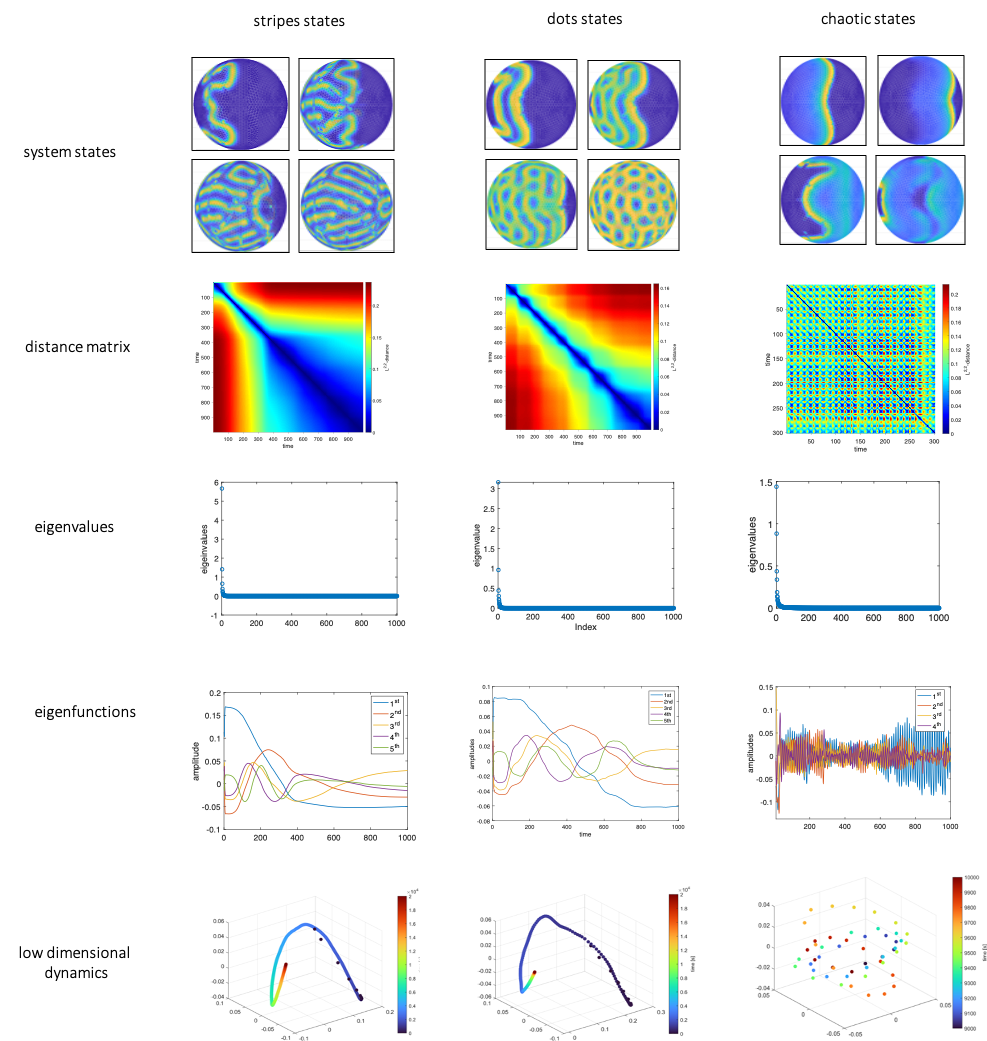}
    \caption{Multidimensional scaling and dimensionality reduction of complex systems dynamics.The modelled solutions to the Gray-Scott model (upper row) for the following parameter choices, from left to right: $A=0.04$, $B=0.0584$ and $A=0.04$, $B=0.062$ and $A=0.01$, $B=0.033$. The lower rows show the corresponding distance matrix embeddings for the scalar field $u$ (middle row) and the vector field $\boldsymbol{\nabla}u$ (bottom row). For the turbulent solution, shorter timescales are used in the embedding plots to show the circular structure.}
    \label{fig:turing_all}
\end{figure}

\begin{figure}[t]
    \centering
    \includegraphics[width=1\textwidth]{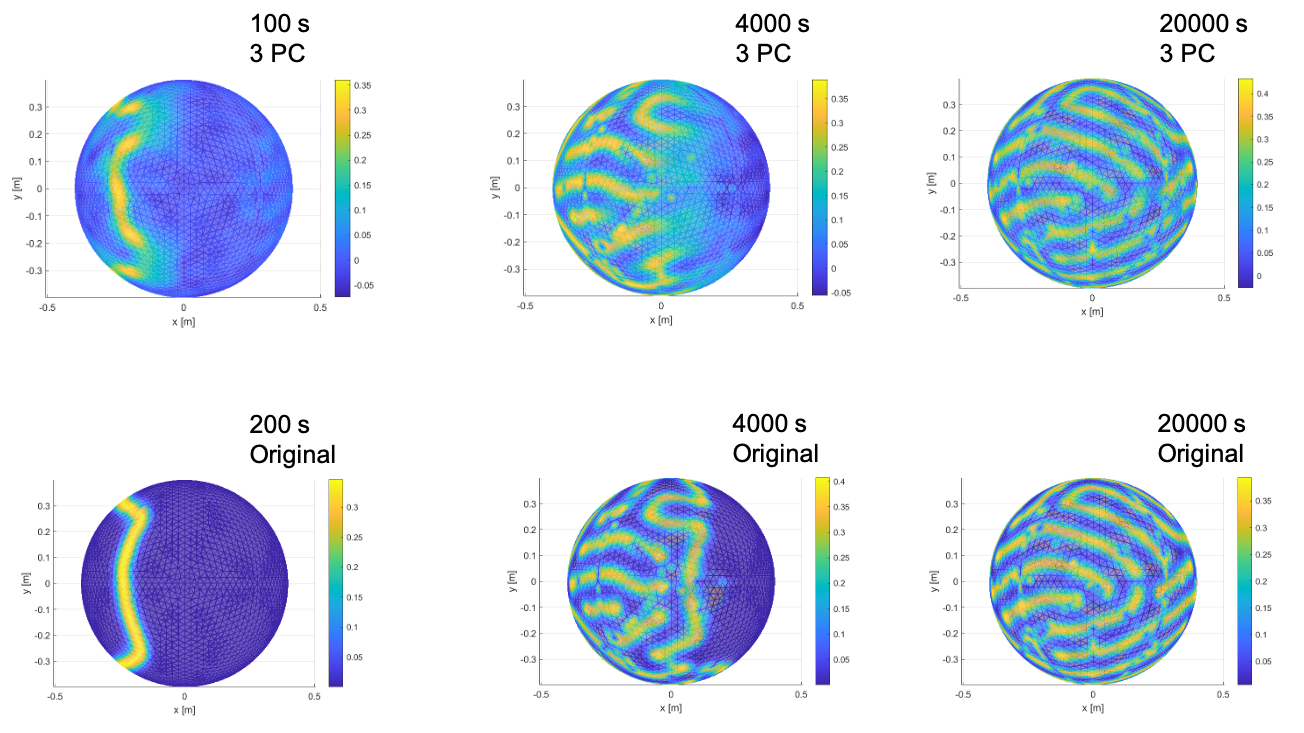}
    \caption{The first three principal components of the striped Turing pattern (above) compared to the original images (below), at 200 s, 4000 s, and 20000 s, from left to right.}
\label{fig:turing_approx}
\end{figure}

\subsection{Dynamical attractors detection} \label{sec:Lya}

Chaotic systems are characterised by their sensitive dependence on initial conditions, that is, the separation between two initially  close trajectories in the state space will grow exponentially fast in time \cite{kantz2004}. 
The framework of vectors over discrete measure spaces allows us to use data of time-evolving systems to visually detect the presence of dynamic attractors, and to determine the nature of these, that is, whether these attractors are chaotic or not. 

Lyapunov exponents have been widely used to detect the presence of attractors in time-dependet data, where each data point is as a vector of $\mathbb R^d$. We extend the use of Lyapunov exponents to analyse data that has the structure of a vector over a discrete measure space. In this case, the vector representation of the data obtained via the multidimensional embedding allows us to make use of this approach. We have computed the Lyapunov exponent in two ways: 1) through the analysis of the first principal component of the multidimensional embedding of the $L^1$-distance matrix of the images and 2) through the $L^2$-norm of the images. The estimated largest Lyapunov exponents for the two methods are 0.38 and 0.15, respectively. We compute the maximal Lyapunov exponent using the built-in function 'lyapunovExponent' in Matlab2022b. The results confirm the results of obtained via the geometric methodology: the turbulent solution with oscillations is the only one with a positive maximal Lyapunov exponent.

\section{Conclusions} \label{sec:conclusions}

The ubiquity of complex systems across different fields of science makes the development of analytic methodologies of great relevance. However, the analysis of high-dimensional spatio-temporal data obtained from complex systems dynamics presents various challenges. Spaces of vector fields of rank over discrete measure spaces provide a robust geometric framework that can be efficiently used as an input in unsupervised learning algorithms to analyse a variety of unstructured data types, including images, gradients of images, functions on meshes and vector-valued functions over simplicial complexes, for pattern recognition. We demonstrated the validity of this methodology to analyse high-dimensional data originating from the study of complex systems. As a result of the dependencies and connections of the entities that form a given complex system, the proposed geometric framework provides a robust low-dimensional model of the system dynamics. We have also shown that our geometric approach provides a visual alternative to the Takens-based methods to detect global dynamic attractors. Two advantages of the proposed methodology are highlighted: first, no prior knowledge of the dynamical system is required; second, it can be used on high-dimensional unstructured data sets, including RGB images with high resolution. Even though our geometric framework, leveraging the theory of vector fields over discrete measure spaces, has been motivated by its applications to soft matter research, its level of generality can benefit a large community of experts involved in the analysis of complex spatio-temporal datasets arising from dynamical systems.

\section*{Data availability}
The data sets generated and analysed during the current study are available from the corresponding author on reasonable request.

\section*{Code availability}
The code used during the current study is available from the corresponding author on reasonable request.

\section*{Acknowledgements}
This work was supported by the Leverhulme Trust (grant RPG-2019-055).

\section*{Authors' contributions}
IMS developed the geometric methodology for data analysis. MvR implemented the numerical solutions of the Ginzburg-Landau equation and the Grey-Scott equation. Both IMS and MvR carried out the data analysis. IMS and MVR wrote the paper with input from all the authors. MK, JB, and GD planned the core research.

\section*{Competing interests}
The authors declare that they have no competing interests.

\bibliography{bibliography}
\bibliographystyle{plain}

\end{document}